\documentclass{article}

\usepackage[final,main,nonatbib]{neurips_2026}
\usepackage[utf8]{inputenc}
\usepackage[T1]{fontenc}
\usepackage{hyperref}
\usepackage{url}
\usepackage{booktabs}
\usepackage{array}
\usepackage{amsmath}
\usepackage{amssymb}
\usepackage{amsfonts}
\usepackage{amsthm}
\usepackage{nicefrac}
\usepackage{microtype}
\usepackage{xcolor}
\usepackage{graphicx}
\usepackage[backend=biber,style=numeric-comp,sorting=none,maxbibnames=99]{biblatex}
\newtheorem{proposition}{Proposition}

\title{Neural-DISCO: Source-Conditioned Counterfactual Editing of Neural Population Activity}

\author{%
  Nico Policzer\thanks{ These authors contributed equally to this work.} \quad
  James Ho\footnotemark[1] \quad
  Joseph Soo\footnotemark[1] \quad
  Xin Tang\thanks{Corresponding author: \texttt{xin.tang@msl.ubc.ca}.} \\
  Michael Smith Laboratories and Department of Computer Science \\
  University of British Columbia \\
  Vancouver, BC, Canada
}

\begin{document}

\maketitle

\begin{abstract}
Understanding how neural activity encodes behavioral and sensory variables is a key challenge in systems neuroscience. This requires models that go beyond predicting activity from labels and instead support source-conditioned counterfactual editing: given a recorded source response, we should be able to edit one measured behavioral or stimulus label and predict how neural activity would change while preserving other observed labels and source-specific variability not explained by the labels. Here, we introduce \textbf{Neural-DISCO}, a disentangled representation learning framework for counterfactual modeling of neural activity. Neural-DISCO partitions the latent space into label-specific components, each associated with an observed behavioral or stimulus label, together with a residual latent that captures variability not supplied through observed labels. Across simulated and real neural recording datasets, we show that this structure enables selective editing of measured labels to generate counterfactual neural responses while preserving the remaining content of the original activity. Further, to gain insight into how behavioral and stimulus labels are encoded at single-neuron resolution, we pair this disentangled framework with feature attribution methods that identify which neurons are most important for each label. We find that disentanglement enhances the recovery of known functional cell roles in both synthetic and real data. Together, these results support Neural-DISCO as a framework for counterfactual modeling of neural activity and provide a practical interface for probing how neural populations encode behavior and sensory information. Code is available at \url{https://github.com/TangLab-UBC/Neural-DISCO-neurips}.
\end{abstract}

\section{Introduction}
Modern neural recording technologies record hundreds to thousands of neurons while animals move, make decisions, and receive structured sensory input. These data make it possible to relate neural population activity to simultaneously recorded behavioral and stimulus labels, a central goal in systems neuroscience \parencite{jun2017fully,steinmetz2019distributed,urai2022large,krakauer2017neuroscience,mathis2026joint,mathis2024decoding,glaser2020machine,Tang2023, international2025brain}. Many questions, however, are not purely predictive. We often want to ask how a neural population response would change if a measured variable were altered---for example, if the animal occupied a different position, reached in another direction, or ran at a different speed---while keeping the remaining source-specific state as intact as possible.

\begin{figure}[!b]
  \centering
  \includegraphics[width=\textwidth]{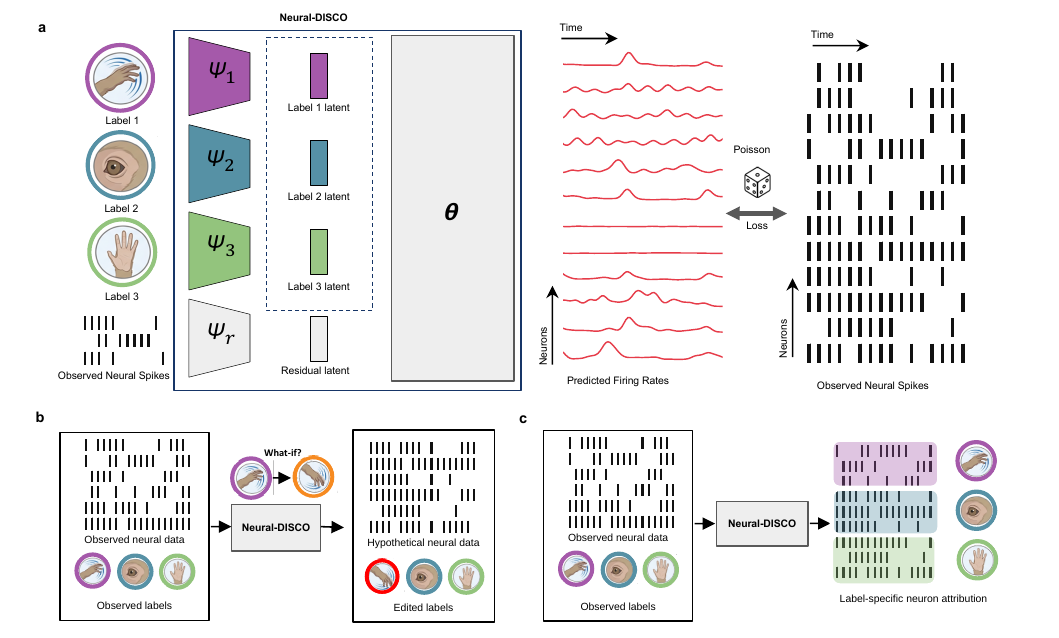}
  \caption{\textbf{Neural-DISCO overview.} (a) Neural-DISCO reconstructs source spike counts from supervised label latents and a bottlenecked residual latent. (b) Counterfactual editing replaces a selected label latent while holding the remaining label blocks and the source residual fixed. (c) The same block-structured generator supports label-specific neuron attribution.}
  \label{fig:method}
\end{figure}
We call this problem \emph{source-conditioned counterfactual editing}, borrowing counterfactual terminology from causal inference while using it operationally rather than claiming causal identification from observational data alone \parencite{Holland01121986}. Given a source response $x$ and observed label values $(y_1,\ldots,y_K)$, the goal is to replace one label value $y_a$ by $\hat y_a$ and generate neural activity for the edited condition while preserving the other observed label values $y_{-a}$, and residual structure not captured by the labels. This differs from ordinary label-conditioned generation \parencite{sohn2015learning,mirza2014conditional}. A model of $p(x'\mid \hat y_a,y_{-a})$ can sample plausible activity consistent with the target labels, but unless it is conditioned on the source response, it generally marginalizes over internal state, spontaneous movement, and unmeasured task variables \parencite{stringer2019spontaneous,musall2019single,urai2022large}. This source-conditioned counterfactual setting instead requires representations that retain residual content of the source sample. We note that only measured labels are editable; unmeasured variation may be preserved in the source residual but is neither identified nor directly editable. 


We introduce \textbf{Neural-DISCO}, a \underline{DIS}entangled learning model for source-conditioned \underline{CO}unterfactual editing of neural population activity (Fig.~\ref{fig:method}). Each measured behavioral or stimulus label is encoded by its own embedding network into corresponding label-specific latents, while a separate residual encoder maps the source spike counts to a residual latent. A generator then decodes the concatenated latents into firing rates. To reduce observed-label leakage into the residual latent, we corrupt the residual with Gaussian noise during training and penalize its activation norm. This bottleneck biases the model to route observed factors through their designated label latents while preserving residual source-specific variability in the residual latent.


At test time, Neural-DISCO encodes a source response, replaces one label-specific latent  with the embedding of a target value, keeps the residual and other label-specific latents fixed, and decodes the edited representation into firing rates. The same factorization also supports neuron-level interpretation. Because labels enter the generator through label-specific blocks, we compute integrated gradients to obtain label-specific attribution scores for each neuron \parencite{intgrad}.




We evaluate Neural-DISCO in three settings that probe complementary failure modes: a rat neural activity simulator with known functional cell classes and ground-truth edited rates; macaque reaching electrophysiology, where edited activity is evaluated by held-out behavioral decoding; and a mouse visual encoding Neuropixels dataset, where running-speed and grating stimulus orientation edits are evaluated jointly with stimulus preservation and independent unit-level tuning metrics. Across these settings, Neural-DISCO improves the balance between edit fidelity and residual preservation relative to classical Generalized Linear Model (GLM) baselines. Further, on traditional conditional-generation tasks Neural-DISCO performs competitively with strong neural encoding baselines.

Our contribution is a source-conditioned framework for neural activity editing and label-specific attribution. Neural-DISCO combines label-specific embeddings with a noisy residual bottleneck characterized by a capacity bound, and experiments on controlled simulation and two real neural recording datasets support selective editing, source preservation, and label-specific neuron attribution.


\section{Related work}
\textbf{Neural latent-variable models and behavior-conditioned generation.}
Latent-variable methods such as GPFA, LFADS, and demixed PCA summarize population activity in low-dimensional representations \parencite{yu2009gaussian,pandarinathInferringSingletrialNeural2018,kobak2016demixed}. Recent generative models, including LDNS, NEDS, pi-VAE, and GNOCCHI, model neural activity conditionally on behavioral labels or learned neural codes \parencite{kapoor2024latent,neds,pivae,mccart2024diffusion}. Neural-DISCO instead targets source-conditioned editing: the edited sample should encode a target label while retaining residual content from a particular source response.

\textbf{Residual neural structure and behaviorally targeted dynamics.}
Single-trial neural variability reflects internal state, spontaneous movement, unmeasured task variables, and recording noise \parencite{stringer2019spontaneous,musall2019single,urai2022large,steinmetz2019distributed,international2025brain}. Related methods identify task-specific subspaces \parencite{aoi2020multidimensional}, separate behaviorally relevant from other neural dynamics \parencite{psid,tndm,ipsid,sani2024dissociative}, or distinguish shared and private neural--behavioral representations \parencite{gondurMultimodalGaussianProcess2026,yi2025sharedae}. Neural-DISCO is complementary: it provides an editable factorization that makes label encodings directly addressable while preserving residual source-specific structure.

\textbf{Disentangled representation learning and counterfactual editing.}
Disentangled representation learning traditionally includes approaches based on constrained VAEs, mutual-information objectives, and factorial latent priors \parencite{higgins2017beta,chen2016infogan,kim2018factorvae}. However, unsupervised disentanglement is not identifiable without inductive biases or supervision \parencite{locatello2019challenging}, motivating our supervised factorization. Counterfactual or attribute editing in generative models often aims to vary selected attributes or labels while preserving source content \parencite{lample2017fader,huang2018munit,lee2018drit}. Our bottleneck design is closest in spirit to LORD and biolord, which use asymmetric noise regularization to separate observed labels from residual latents \parencite{lord,biolord}. Neural-DISCO adapts this principle to neural spike counts with a Poisson observation model, multiple behavioral or stimulus labels, source-conditioned editing, and neuron-level attribution.

\textbf{Neuron-level interpretation.}
Classical neuroscience interprets neurons through tuning and modulation analyses relating firing rates to stimuli or behavioral state \parencite{hubel1962receptive,georgopoulos1982relations,georgopoulos1986neuronal,niell2010modulation}. Recently, xCEBRA provides a complementary approach to identifiable attribution maps for neural time series \parencite{xcebra}. BrainDiVE synthesizes images predicted to activate cortical regions \parencite{luo2023braindive,henderson2026diffusion}, while MIG-Vis visualizes the stimulus semantics of neural latent groups through diffusion \parencite{wang2026migvis}. These methods produce stimulus images; Neural-DISCO predicts edited neural responses while retaining source residual structure.


\section{Methods}
\label{sec:methods}
\subsection{Neural-DISCO model}
\label{sec:model}
\textbf{Data and notation.}
Let $\mathcal D=(X,Y)$ be paired neural and behavioral or stimulus observations. Each source response $x \in X \subseteq \mathbb{R}^n$ is a vector of binned spike counts from $n$ neurons, and the paired label vector $y=(y_1,\ldots,y_K)$ contains the $K$ measured labels recorded in the same bin. A label may be scalar, vector-valued, or categorical. The editing task is: given a source response $x$ and labels $y$, replace one or more selected labels with user-specified target values while preserving unedited labels and source-specific variability not represented by the observed labels. For an edit to label $a$, $\hat y_a$ denotes its target value and $y_{-a}=(y_b)_{b\ne a}$ denotes all other labels.

\textbf{Factor-specific and residual representations.}
For each observed label $a\in\{1,\ldots,K\}$, Neural-DISCO uses a separate embedding network
\begin{equation}
    z_a = \phi_a(y_a), \qquad z_a\in\mathbb R^{d_a},
\end{equation}
where $\phi_a:\mathcal Y_a\to\mathbb R^{d_a}$ is learned. Because $z_a$ is computed only from $y_a$, it can carry only information available through that observed label, including variables statistically correlated with $y_a$ \parencite{locatello2019challenging}. Idiosyncratic source content not determined by observed labels must pass through the residual path.


A residual encoder maps the source response to the source residual latent $u = \phi_{\mathrm{res}}(x), u\in\mathbb R^m$.
The clean representation is
\begin{equation}
    z=(z_1,\ldots,z_K,u), \qquad d_{\mathrm{tot}}=\sum_{a=1}^K d_a+m.
\end{equation}
During training only, the residual latent is corrupted by Gaussian noise,
\[
\tilde u = u+\epsilon,\qquad \epsilon\sim\mathcal N(0,\sigma^2 I_m),
\qquad
\tilde z=(z_1,\ldots,z_K,\tilde u).
\]
At test time and for counterfactual editing, we use the clean residual
\(u=\phi_{\mathrm{res}}(x)\); the noisy residual \(\tilde u\) is used only to
regularize the residual channel during training.

\textbf{Observation model.}
The label-specific encoders $\phi_a$, residual encoder $\phi_{\mathrm{res}}$, and generator $G_\theta$ are implemented as feed-forward multilayer perceptrons (MLPs), with dataset-specific architectures summarized in Appendix Table~\ref{tab:hyperparams}. The generator $G_\theta:\mathbb R^{d_{\mathrm{tot}}}\to\mathbb R_{>0}^n$ maps the concatenated representation to rates,
\begin{equation}
    r_\theta(\tilde z)=G_\theta(\tilde z),
\end{equation}
with positivity enforced by a softplus output. Conditional on $\tilde z$, neurons are modeled with an independent Poisson likelihood \parencite{pandarinathInferringSingletrialNeural2018, truccoloPointProcessFramework2005},
\begin{equation}
p_\theta(x\mid \tilde z)=\prod_{\ell=1}^n \mathrm{Poisson}\big(x_\ell\mid r_{\theta,\ell}(\tilde z)\big).
\end{equation}


\textbf{Training objective.}
The model minimizes the Poisson negative log-likelihood plus a residual activation penalty, dropping the additive constants \(\sum_{\ell=1}^n \log(x_\ell^{(s)}!)\) that do not depend on model parameters.
For sample \(s\), define
\[
u^{(s)}=\phi_{\mathrm{res}}(x^{(s)}),\qquad
\tilde u^{(s)}=u^{(s)}+\epsilon^{(s)},\qquad
\tilde z^{(s)}
=
\big(\phi_1(y_1^{(s)}),\ldots,\phi_K(y_K^{(s)}),\tilde u^{(s)}\big),
\]
and \(r_\theta^{(s)}=G_\theta(\tilde z^{(s)})\). The model minimizes
\begin{align}
\mathcal L(\theta,\{\phi_a\},\phi_{\mathrm{res}})
=&\frac{1}{N}\sum_{s=1}^N \mathbb E_{\epsilon^{(s)}}
\left[
\sum_{\ell=1}^n
\left(
r_{\theta,\ell}^{(s)}
-
x_\ell^{(s)}\log r_{\theta,\ell}^{(s)}
\right)
\right]
+\lambda\frac{1}{N}\sum_{s=1}^N\|u^{(s)}\|_2^2 .
\label{eq:training_objective}
\end{align}
The expectation over noise is estimated with one Monte Carlo sample of \(\epsilon^{(s)}\) per training example in each minibatch.
The $\lambda$-selection criterion favors high validation edit fidelity and low observed-label decodability from the clean residual, averaged across ten validation folds. Dataset-specific settings are reported in Appendix~\ref{app:details}.
    

\paragraph{Capacity interpretation of the residual bottleneck.}
The residual should preserve source information needed for reconstruction while avoiding redundant storage of the observed labels. We adapt the asymmetric-noise idea in \parencite{lord, tishbyInformationBottleneckMethod2000, alemiDeepVariationalInformation2016}: observed labels enter deterministically through label-specific latents, whereas the residual path is made noisy and norm-penalized.

\begin{proposition} \textbf{Capacity of the noisy training-time residual channel.}
\label{prop:noisy-residual-capacity} Fix a trained residual encoder $\phi_{\mathrm{res}}$ and assume $\sigma>0$.
Let $S\sim \mathrm{Unif}\{1,\ldots,N\}$ be the training-sample index, and define
\[
    X_S=x^{(S)},\qquad
    Y_{a,S}=y_a^{(S)},\qquad
    U_S=\phi_{\mathrm{res}}(X_S)\in\mathbb R^m .
\]
Assume \(m\geq 1\) and let
\(\epsilon\sim\mathcal N(0,\sigma^2 I_m)\) be independent of \(S\). Define the
noisy training-time residual and empirical residual power by
\[
    \tilde U = U_S+\epsilon, \qquad
    \hat P
    =
    \frac{1}{N}
    \sum_{s=1}^N
    \|\phi_{\mathrm{res}}(x^{(s)})\|_2^2 .
\]
Then, using natural logarithms so that mutual information is measured in nats,
\begin{equation}
    I(S;\tilde U)
    \leq
    \frac{m}{2}
    \log\left(
        1+\frac{\hat P}{m\sigma^2}
    \right).
    \label{eq:capacity_bound}
\end{equation}
Moreover, for any observed factor \(Y_{a,S}\) and deterministic embedding
\(Z_{a,S}=\phi_a(Y_{a,S})\),
\begin{equation}
    I(Y_{a,S};\tilde U)
    \leq
    I(S;\tilde U),
    \qquad
    I(Z_{a,S};\tilde U)
    \leq
    I(Y_{a,S};\tilde U).
    \label{eq:capacity_dpi}
\end{equation}
\end{proposition}
The proof is given in Appendix~\ref{app:proof}. Proposition~\ref{prop:noisy-residual-capacity}
applies only to the randomized residual channel \(S\mapsto\tilde U\) used during
training, and does not directly bound the clean residual \(S\mapsto U_S\) which is used at test time. The result does not guarantee disentanglement, but formalizes why training with the information bottleneck reduces the information capacity available for residual leakage. Decoding probes show that each label is less decodable from the residual when supplied to the model rather than when omitted as intended (Appendix Table~\ref{tab:allen_clean_residual}).
Identifiability results rely on structured assumptions, such as auxiliary-variable-conditioned factorial priors \parencite{khemakhem2020variational} or paired observations before and after interventions \parencite{brehmer2022weakly}. The asymmetric-noise bottleneck \parencite{lord} controls a noisy training channel without providing such identification guarantees.

\subsection{Source-conditioned counterfactual editing}
\label{sec:editing}

Given a source pair $(x,y)$, Neural-DISCO computes
\begin{equation}
    z_a=\phi_a(y_a),\quad a=1,\ldots,K,\qquad u=\phi_{\mathrm{res}}(x).
\end{equation}
To edit label $a$, only that label-specific latent is replaced,
\begin{equation}
    \hat z_a=\phi_a(\hat y_a), \quad \hat z =(z_1,\ldots,z_{a-1},\hat z_a,z_{a+1},\ldots,z_K,u).
\end{equation}
The counterfactual rate vector is
\begin{equation}
\hat r=G_\theta(\hat z).
\end{equation}
When sampled spikes are required,
\begin{equation}
    \hat x_\ell\sim\mathrm{Poisson}(\hat r_\ell),\qquad \ell=1,\ldots,n.
\end{equation}

\paragraph{Counterfactual semantics.}
We use ``counterfactual'' operationally. Outputs should be interpreted as model-based edits, not as identified causal effects from observational recordings. Poisson sampling after label latent replacement yields the corresponding interventional distribution under three assumptions: the true Poisson spike generator factorizes as in Appendix~\ref{cf_semantics}; interventions leave the unedited labels and source state fixed; and the model recovers the true label maps, generator, and source-state encoding.


\subsection{Neuron-level label attribution}
\label{sec:attribution}
We quantify label-specific contributions to each neuron's pre-softplus generator output using integrated gradients \parencite{intgrad}, along a straight path from zeroed observed labels and a zero residual to the source input. We first average signed attributions across 32 randomly selected source samples, then take the mean absolute value over each label's input dimensions or the residual dimensions. Appendix~\ref{ig} gives the precise definition.


\subsection{Baselines and ablations}
\label{sec:baselines}
We compare with a no regularization baseline (w/o reg.), a zero residual baseline, pi-VAE \parencite{pivae}, and NEDS \parencite{neds}, and Gaussian and Poisson GLMs. The GLMs linearly project the observed labels to predict neural activity under their corresponding distributions. The Gaussian GLM is included as a linear editing baseline; we use it to edit by replacing the projection of the source label with the projection of the target label. The zero residual control removes Neural-DISCO's source residual during editing. The no regularization baseline is an architecture-matched ablation without any information bottleneck in the residual ($\lambda=\sigma=0$). pi-VAE, NEDS and the Poisson GLM do not expose the same source-conditioned residual-editing interface. For counterfactual evaluation, we therefore use their target-conditioned generation interface with the edited labels. These comparisons test whether source-conditioned residual preservation provides benefits beyond strong target-conditioned baselines; they should not be read as claiming that these baselines were designed for the same editing interface. 

For the Allen Visual Coding evaluation, speed editors share the training-fold mean-speed target; methods within each task share decoders trained on real training activity. Checkpoints are selected using validation folds. Scores from the held-out test epoch are averaged across folds, then across sessions for the 32-session analysis.

\section{Experiments}
\label{sec:experiments}
We evaluate whether Neural-DISCO supports source-conditioned counterfactual editing, as well as standard label-conditioned generation. The experiments address three questions. First, when the data-generating factors and counterfactual firing rates are known, does the model change the requested label while preserving non-target and unobserved structure? Second, in real motor cortical recordings without ground-truth counterfactual spikes, does generated activity decode the target behavioral label rather than remain source-like? Third, in a larger visual cortex dataset, can the same factorization average out a specified behavioral label or edit the encoding of a specific stimulus label, while preserving held-out stimulus and behavioral information, and producing unit-level attributions aligned with independent biological metrics?

For our non-simulated datasets, there is no ground-truth neural response for performed edits. We therefore use an indirect decoding-based evaluation strategy. We use decoders trained only on real neural activity data to evaluate whether generated activity encodes both the target label, as well as labels unobserved by the model which should be conserved through the residual latent.

Unless otherwise stated, values are means across ten folds; $\pm$ values and error bars denote the standard error of the mean (SEM). For our Allen Visual Coding dataset editing, we report the mean and SEM across 32 sessions. Appendix~\ref{app:details} gives preprocessing and training settings; Appendices~\ref{app:mixed_selectivity} and~\ref{app:data_requirements} report mixed-selectivity and sampling controls.

\begin{figure}[!b]
  \centering
  \includegraphics[width=\textwidth]{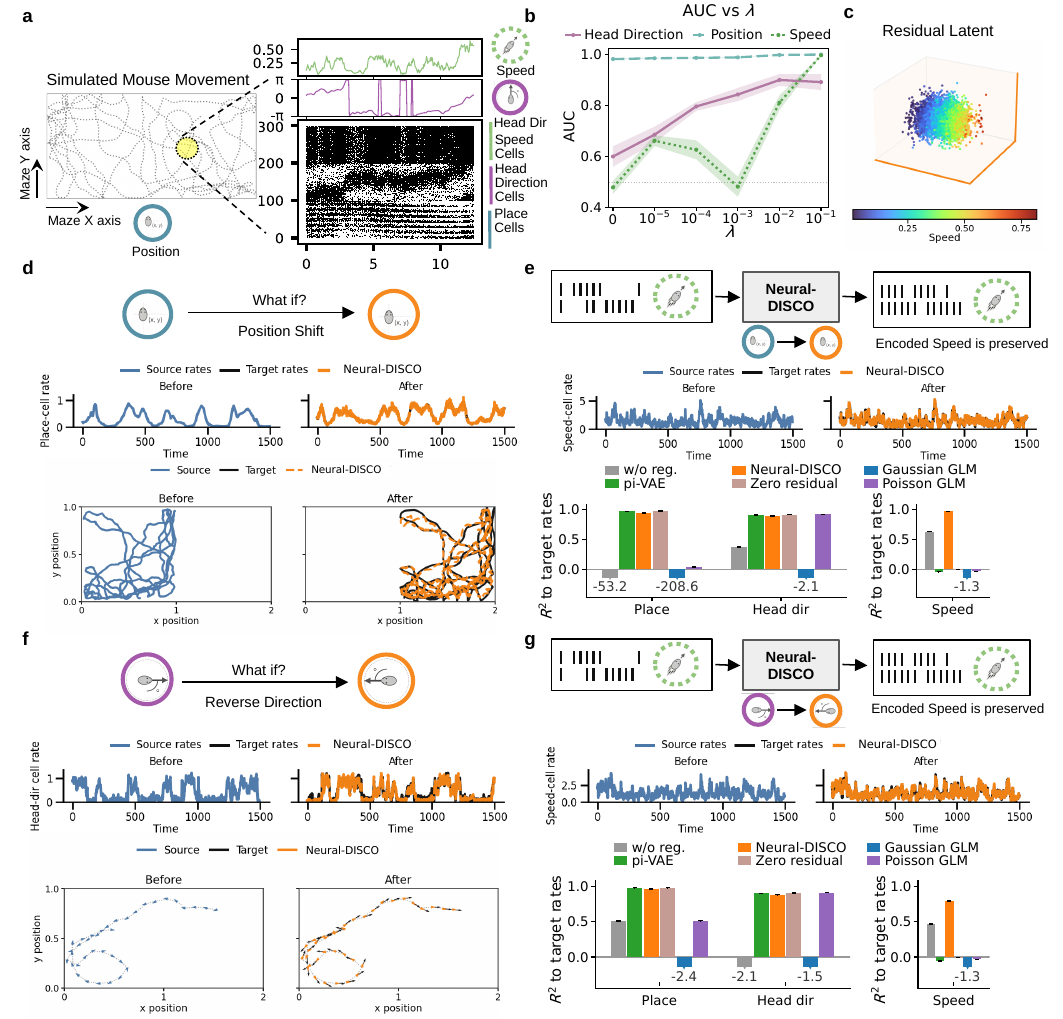}
  \caption{\textbf{Controlled RatInABox evaluation.} (a) Simulated movement with place, head-direction and speed cells. (b) Cell-type recovery across $\lambda$; speed-cell recovery uses residual attribution because speed is withheld. (c) Residual PCA colored by speed. (d,e) Position shift while retaining head direction and speed. (f,g) Head-direction reversal while retaining position and speed. Generated neural activity is scored against simulator targets.}
  \label{fig:rib}
\end{figure}
\subsection{RatInABox simulation}
\label{sec:rib}
We first evaluate Neural-DISCO on a controlled RatInABox simulation dataset, where an agent moves in a $2\times1$ meter enclosed environment and both functional cell classes and ground-truth counterfactual firing rates are known \parencite{ratinabox}. The population contains 100 place cells \parencite{o1971hippocampus}, 100 head-direction cells \parencite{taube1990head}, and 100 speed cells \parencite{kropff2015speed}, whose rates depend on position, head-direction and speed respectively. Neural-DISCO receives position and head direction as observed labels, while speed is withheld and can only be represented through the residual latent. This creates a direct test of both label-specific editing and residual preservation.


\textbf{Attribution recovers simulated cell classes.}
For each label, we use the corresponding neuron-level attribution score to classify the matching cell type and report area under the ROC curve (AUC). Position attribution should identify place cells, head-direction attribution should identify head-direction cells, and residual attribution should identify speed cells as speed is unobserved. Cell-type recovery is higher at the strongest bottleneck than at $\lambda=0$ (Fig.~\ref{fig:rib}b). The residual representation is also organized by held-out speed (Fig.~\ref{fig:rib}c), consistent with the residual pathway preserving source-specific structure not supplied by the observed labels.

\textbf{Counterfactual edited rates match ground-truth targets.}
The simulator allows direct comparison between generated rates and the true rates under an edited behavioral label state. We evaluate two edits. In the position-shift edit, the source trajectory is confined to the left half of the environment and the position label is shifted to the right half, while head direction and speed are unchanged. The desired output should match target place-cell rates, while preserving source speed-cell and head-direction cell rates. Neural-DISCO follows the shifted place-cell target while preserving speed-cell activity, whose corresponding behavior is unobserved, through the residual latent (Fig.~\ref{fig:rib}d,e). In the head-direction edit, head direction is reversed while position and speed are unchanged. The model should change head-direction-cell activity, preserve place-cell activity, and preserve speed-cell activity through the residual. Neural-DISCO again tracks the edited target rates for the manipulated label while retaining non-target and held-out components (Fig.~\ref{fig:rib}f,g). The no regularization model without the bottleneck shows the expected failure mode: the residual pathway can carry label information, so replacing an observed-label embedding does not reliably make the decoded activity target-like.
As an additional diagnostic, we train variants in which each behavioral variable is withheld in turn and must be represented through the residual. These diagnostic tests yield the same conclusion. Appendix~\ref{rat_witheld} reports these withheld-factor tests; Appendix~\ref{app:rat_unseen_combinations} tests generalization when position--heading combinations are withheld during training.

\begin{figure}[!ht]
  \centering
  \includegraphics[width=\textwidth]{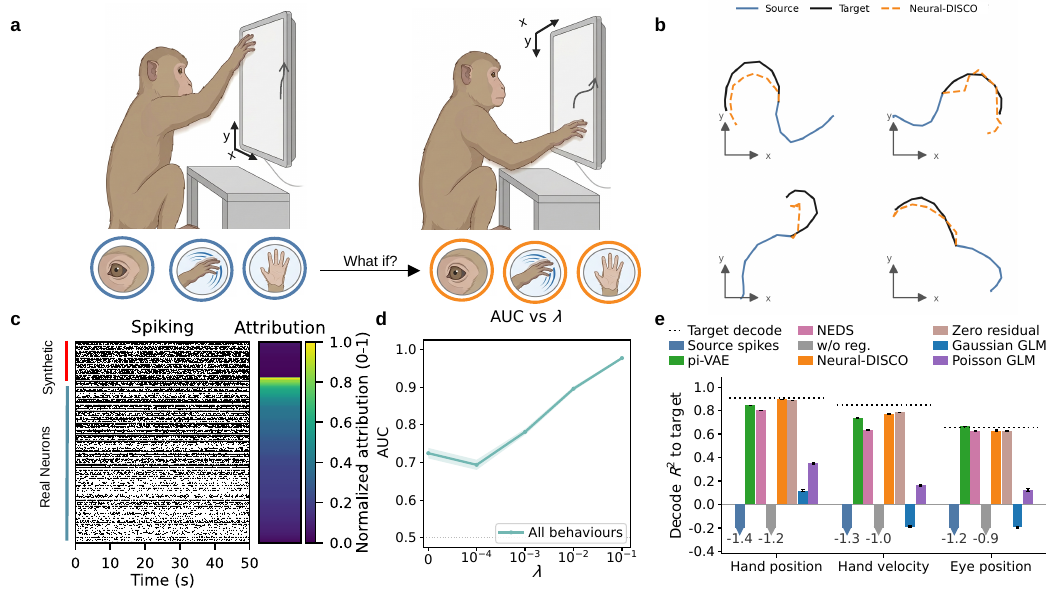}
  \caption{\textbf{Counterfactual editing in macaque reaching electrophysiology.} (a) Editing replaces hand-position, hand-velocity and eye-position blocks while retaining the source residual. (b) Trajectories decoded from source spikes, target spikes and Neural-DISCO counterfactual activity. (c,d) Synthetic constant-rate controls test whether attribution distinguishes recorded neurons from controls. (e) Target-decoding $R^2$ comparisons; dotted lines show the empirical target-spike decoding reference.}
  \label{fig:mcmaze}
\end{figure}
\subsection{Macaque reaching electrophysiology}
\label{sec:mcmaze}

We next evaluate Neural-DISCO on real motor and premotor cortical recordings from a center-out reaching task in the Neural Latents Benchmark \parencite{pei2022neurallatentsbenchmark21}. The dataset contains 137 recorded units. Unlike the simulator, the recordings contain unobserved variability, recording noise, and correlated behavioral structure, and there is no ground-truth neural response for a hypothetical edited reach. We therefore evaluate edited activity through behavioral decoding rather than direct rate error. We also benchmark label-conditioned generators and use inconsistent behavioral targets to check whether decoding simply tracks the requested labels (Appendix~\ref{app:implausible_targets}).

\textbf{Edited activity decodes as the target reach.}
For each held-out source trial, we select a target trial with matched task structure and a different reach trajectory. Here we jointly replace the hand-position, hand-velocity, and eye-position latents with target-trial label latents to test whether the generator can produce a target-like behavioral episode while retaining the source residual. A decoder based on CEBRA \parencite{cebra} trained only on real neural spikes is applied to source spikes, target spikes, and generated activity. We compute target-decoding $R^2$ against the target behavior. Source-spike decoding is a negative control, and target-spike decoding is an empirical reference for the performance achieved by this decoder on real data.

As a result, Neural-DISCO counterfactuals decode close to target trajectories, with $R^2=0.895\pm0.002$ for hand position, $0.772\pm0.003$ for hand velocity, and $0.630\pm0.004$ for eye position. Target-spike decoding reaches $0.909$, $0.845$, and $0.656$, respectively, whereas source spikes decode poorly against the target trajectories. The unregularized ablation remains source-like and yields negative target-decoding $R^2$, consistent with label information leaking through the residual pathway (Fig.~\ref{fig:mcmaze_no_reg_decode_r2}). Neural-DISCO substantially improves over the unregularized control and performs competitively with pi-VAE and NEDS across decoded variables (Fig.~\ref{fig:mcmaze}e). The decoded examples in Fig.~\ref{fig:mcmaze}b show the same qualitative behavior: generated trajectories follow the target reach rather than the source reach, while the unregularized ablation remains source-like (Fig.~\ref{fig:mcmaze_no_reg_reach_trajectories}).

\textbf{Synthetic control neurons provide an attribution sanity check.}
This dataset does not provide ground-truth behavioral importance labels for individual units. We therefore add synthetic constant-rate control neurons as a negative-control population. These controls are Poisson spike trains with constant behavior-independent firing rates matched to the dataset mean firing rate. They should receive low attribution to behavior, whereas recorded motor and premotor units are more likely to contain movement-related structure \parencite{georgopoulos1982relations,georgopoulos1986neuronal}. We use the total attribution to all three behaviors in our analyses. Real-neuron recovery AUC increases with $\lambda$ over the positive range (Fig.~\ref{fig:mcmaze}d). This pattern is consistent with the counterfactual results: bottlenecking the residual discourages leakage of observed behavioral information and makes label-specific attribution more accurate. This is a sanity check rather than biological ground truth, because not every recorded unit is expected to be modulated by the observed labels.

\begin{figure}[!b]
  \centering
  \includegraphics[width=\textwidth]{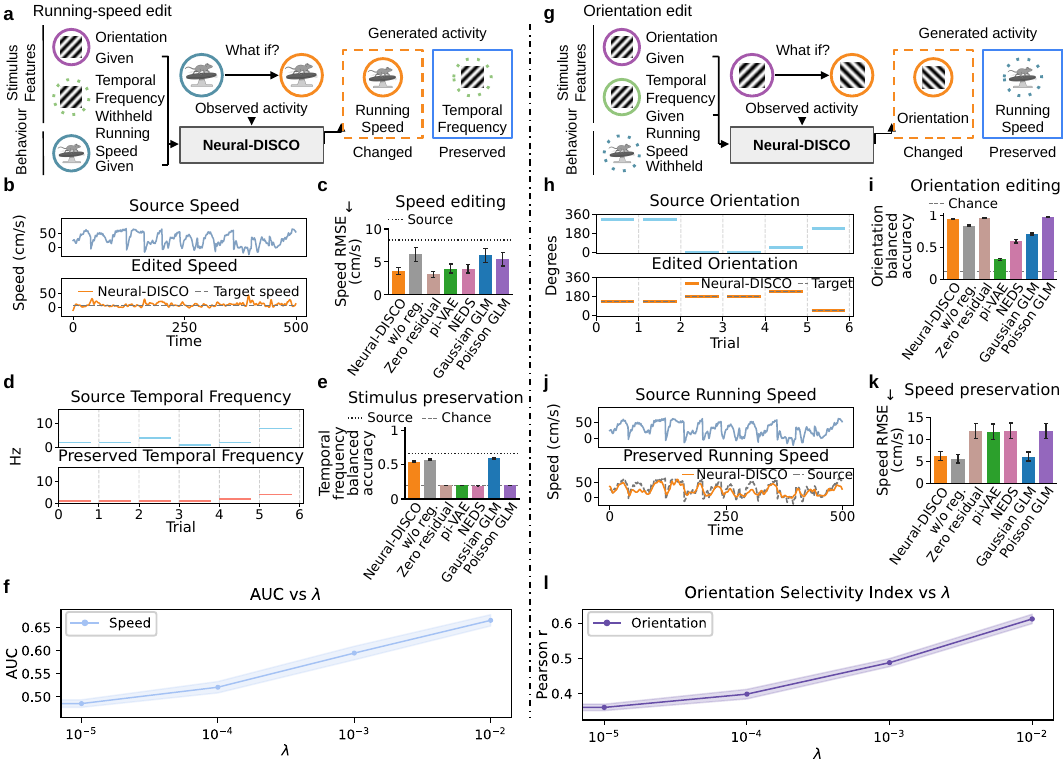}
  \caption{\textbf{Allen Visual Coding evaluation.} (a,g) Speed- and orientation-edit schematics. (b,c) Speed-edit example and speed-target RMSE. (d,e) Temporal-frequency preservation example and balanced accuracy after speed editing. (h,i) Orientation-edit example and target-orientation balanced accuracy. (j,k) Running-speed preservation example and speed RMSE after orientation editing. Editing summaries are mean $\pm$ SEM across 32 session means, each averaging ten folds. (f,l) Running-modulation AUC and orientation-selectivity correlation across $\lambda$. In summary bars, Source denotes decoding of recorded activity.}
  \label{fig:allen}
\end{figure}
\subsection{Allen Visual Coding}
\label{sec:allen}

Finally, we evaluate Neural-DISCO on 32 recording sessions from the Allen Visual Coding Neuropixels dataset \parencite{siegleSurveySpikingMouse2021,AllenSDK_Neuropixels}. These recordings sample multiple visual and related brain areas while mice view visual stimuli and run on a wheel, a behavioral variable known to modulate visual cortical responses \parencite{niell2010modulation}. We focus on drifting-grating stimulus epochs. 

\textbf{Speed editing and stimulus preservation.}
We use running speed and grating orientation as observed labels. The temporal frequency (TF) of the gratings is intentionally withheld from the observed-label embeddings, so temporal-frequency information must be preserved through the residual pathway. We edit running speed to its mean in each training fold, using decoders trained only on real training activity. Across 32 sessions, Neural-DISCO reduces speed-target RMSE to $3.629\pm.567$ cm/s, compared with $6.017\pm1.044$ for Gaussian GLM (Fig.~\ref{fig:allen}c). Models and decoders are fit using ten-fold cross-validation within each session; these values summarize session means, with SEM across sessions.

We then test whether information not supplied as an observed label is preserved after the speed edit. Temporal frequency is withheld during Neural-DISCO training, and a logistic-regression decoder trained on real activity predicts the source temporal frequency from edited activity. Across the same 32 sessions, temporal-frequency balanced accuracy is $.541\pm.013$ for Neural-DISCO and $.597\pm.012$ for Gaussian GLM, both above chance (Fig.~\ref{fig:allen}e). Neural-DISCO achieves lower speed-target RMSE, whereas Gaussian GLM preserves more temporal-frequency information (Appendix Table~\ref{tab:allen_edit_preservation}).

\textbf{Orientation editing and stimulus preservation.}
Across the same 32 sessions, we train a model with orientation and TF observed while running speed is withheld. We edit by rotating orientation by $180^\circ$. Target-orientation and source-TF balanced accuracies are $.949\pm.006$ and $.823\pm.012$, versus $.714\pm.016$ and $.532\pm.015$ for Gaussian GLM (Fig.~\ref{fig:allen}i; Appendix~\ref{app:allen_orientation}; mean $\pm$ SEM across session means). The unregularized ablation has lower orientation fidelity and TF preservation than Neural-DISCO. Figures~\ref{fig:allen}j,k show source-speed preservation; Appendix~\ref{app:allen_orientation} reports all seven methods and additional stimulus edits.

\textbf{Attributions align with dataset-provided unit metrics.}
The Allen dataset provides unit metrics that can be used as independent proxies for attribution quality. We classify units with significant running modulation using the running-modulation $p$-value, thresholding at $p<0.05$, and compute the AUC obtained from running-attribution scores. For stimulus tuning, we compare orientation-attribution scores with the global orientation selectivity index (gOSI) \parencite{AllenSDK_Neuropixels}. As the residual bottleneck is strengthened, running-attribution AUC and orientation-attribution correlation both increase (Fig.~\ref{fig:allen}f,l).

\section{Discussion}
\label{sec:discussion}
Neural-DISCO provides a source-conditioned interface for counterfactual editing of neural population activity. The key design choice is a supervised factorization with a bottlenecked residual: observed labels remain directly addressable, while source-specific variability that is not supplied through those labels can still be preserved. This design targets two common failure modes. A purely label-conditioned generator cannot preserve residual structure from source spiking, whereas an unconstrained residual pathway can leak the edited label and prevent the edit from taking effect \parencite{lord}.


The experiments support this edit--preserve trade-off. In RatInABox, where edited rates and functional cell classes are known, Neural-DISCO edits position or head direction while preserving withheld speed, and its attributions recover simulated cell classes. In macaque reaching recordings, generated activity decodes as the target reach, whereas the unregularized ablation remains source-like. In the multi-session Allen Visual Coding dataset, Neural-DISCO is able to edit orientation and speed while preserving unobserved labels. Further, its label attributions align with dataset-provided unit metrics. 

%

\textbf{Neural-DISCO for Neuroscience.}
Neuronal responses to stimuli are not determined solely by the presented stimulus, but are modulated by latent aspects of brain state such as arousal, locomotion, and attention \parencite{niell2010modulation, pakan2018action, mcginley2015waking}. Neural-DISCO addresses this by providing a flexible editing framework in which the relationship between observed labels and neural activity is conditioned on the preserved source-specific residual latent. Two source responses with identical labels can therefore yield different counterfactual edits if they differ in their residual neural state. We believe this capability can make Neural-DISCO a useful tool for neuroscientists seeking to understand how both observed and unobserved aspects of brain state modulate neural representations. We explore this on the Allen Visual Coding dataset by reversing the editing operation: we keep the label inputs fixed, replace a source response's residual with that of another response with the same orientation and similar running speed, and measure how each neuron's predicted orientation tuning changes (Appendix~\ref{app:hypothesis_generation}). Across neurons, we find that the variability in Neural-DISCO's predicted firing rates over different source residuals recovers the extent to which temporal frequency, which is withheld from the model, modulates each neuron's firing. This proof of concept demonstrates that Neural-DISCO counterfactuals can expose testable hypotheses about how unobserved factors of brain state modulate neural responses.

\textbf{Limitations.}
Neural-DISCO counterfactuals are model-based predictions under a learned factorization, not causal effects from observational data alone \parencite{Holland01121986}. As in prior counterfactual modeling \parencite{sauer2021counterfactualgenerativenetworks}, interpretations are most reliable within the training support and when edited labels can vary independently; edits violating dependencies among behavior, stimuli, and internal state should be treated as extrapolations. Attribution scores likewise measure model sensitivity rather than proving causal or exclusive encoding by individual neurons. The residual bottleneck can introduce a trade-off: a weak residual channel may harm reconstruction and preservation of unobserved labels, whereas an overly strong residual channel can leak observed-label information and make edits ineffective. Our capacity bound and sensitivity analyses characterize this trade-off but do not guarantee disentanglement, especially under correlated labels, missing variables, or dataset shift, consistent with known limitations of disentanglement without sufficient inductive bias or supervision \parencite{locatello2019challenging}. The model is static, without temporal or transition dynamics. Finally, in real neural data, validation remains indirect because decoder-based assays cannot replace ground-truth edited responses from perturbation experiments, matched repeated trials, or more realistic simulators.

\textbf{Broader Impact.} Neural-DISCO is intended for \textit{in-silico} hypothesis generation in computational neuroscience, using available neural recordings. By exposing editable behavioral and stimulus labels while preserving source-specific residual structure, it provides a model-based interface for prioritizing hypotheses about neural population coding before pursuing more costly targeted analyses or perturbation experiments.

\begin{ack}
This work was supported by start-up funding from the University of British Columbia and a Discovery Grant from the Natural Sciences and Engineering Research Council of Canada [RGPIN-2025-04460] to X.T. This research was undertaken, in part, thanks to funding from the Canada Research Chairs Program [CRC-2024-00164]. Research infrastructure used in this work was supported by the Canada Foundation for Innovation through the John R. Evans Leaders Fund [project 45547] and by the British Columbia Knowledge Development Fund. This research was also enabled in part by support provided by the Digital Research Alliance of Canada.

\textbf{Competing interests.} The authors declare no competing interests.
\end{ack}

\section*{References}
{\small
\printbibliography[heading=none]
}

\appendix
\section{Additional experimental and implementation details}
\label{app:details}

Code and reproduction instructions:\\
\url{https://github.com/TangLab-UBC/Neural-DISCO-neurips}.

\subsection{Additional RatInABox Figures}
\label{rat_witheld}
\textbf{Withheld Behavior Flows Through Residual}
We perform a diagnostic test to verify that held-out behavior is routed through the residual latent. We train three models on our RatInABox dataset, each receiving two observed labels while the third behavioral variable is held out. The three settings correspond to withholding position, head direction, or speed. In Fig.~\ref{fig:withheld_latent}, we visualize the residual latent across all models and find that the residual is consistently organized by the withheld behavior, while the observed behaviors are not strongly separated. In Fig.~\ref{fig:rib_withheld_decoding}, we quantify the decodability of each label from the residual latent, again showing that the withheld behavior is more decodable from the residual than the observed behaviors. Together, these support our findings that the residual latent carries information primarily about unobserved source-specific factors.

\textbf{RatInABox Attribution Heatmaps.}
In Fig.~\ref{fig:rib_heatmap}, we plot attribution heatmaps learned by our model for all RatInABox neurons. We include both the standard model, and the unregularized model without the noisy residual bottleneck. We observe the expected block structure: in the standard model, position, head-direction, and residual attributions are concentrated on place, head-direction, and speed cells, respectively. In contrast, the unregularized residual control shows broader residual attribution and weaker factor specificity.

\textbf{Example Counterfactual Rates across Cell Types}. In Fig.~\ref{fig:rib_rates} we plot the firing rates of three example cells, one from each cell type, before and after both counterfactual edits. In the position-shift edit, the model predicts the target place-cell rate while maintaining the source head-direction and speed-cell rates. In the head-direction edit, generated rates track the target head-direction response while retaining place- and speed-cell structure. All three cell types are compared with their simulator target rates.

\textbf{Ablated Model Position Shift Decoding} In Fig.~\ref{fig:rat_position_shift_ablation}, we plot the source and target position shift movement trajectories, alongside our model's counterfactual prediction and our ablated model's counterfactual prediction. To extract trajectories from our model's predictions, we train a position decoder on real neural activity, and then decode from both model's predictions. While our model successfully encodes the target trajectory, the ablated model decoded trajectory looks much less similar to the target. Quantitative decoding results for this evaluation are in Fig.~\ref{fig:rib}e.

\subsection{Additional Macaque Reaching Figures}
In Fig.~\ref{fig:mcmaze_no_reg_decode_r2} we plot the decoding $R^2$ of our ablated model against both source and target trajectories for our macaque reaching edit task. Across behaviors, we observe that the ablated models' encoded neural activity remains source like - decoding each behavior reveals high $R^2$ against the original behaviors, while low, negative $R^2$ against the edit targets. This is further visible in Fig.~\ref{fig:mcmaze_no_reg_reach_trajectories}, where we show four example trajectories including the original, target, and ablated model prediction through decoding. Again, we observe that the ablated model encodes the source trial's trajectory, failing to perform the counterfactual edit.  This further demonstrates that without the noisy information bottleneck from regularization, editing can fail - likely due to behavior specific information flowing through the residual.

\subsection{Additional Allen Visual Coding Figures}
\textbf{Orientation Attribution Heatmaps} In Fig.~\ref{fig:allen_heatmap}, we plot the orientation attribution for all units in session 756029989, sorted by general orientation selectivity index (gOSI) for both regularized and un-regularized models. We see a trend with higher attributed units having greater gOSI values. Colour corresponds to normalized attribution, whose calculation is detailed in Appendix~\ref{ig}.

\subsection{Integrated Gradients}
\label{ig}
For a trained model and clean representation $z=(z_1,\ldots,z_K,u)$, let $g_\ell(z)$ be neuron $\ell$'s generator output before the softplus rate transformation. We compute Integrated Gradients \parencite{intgrad} using a joint baseline over labels and the residual latent,
\[
b = (b_{y_1}, \ldots, b_{y_K}, b_{\mathrm{res}})
\]
 where $b_y=(b_{y_1},\ldots,b_{y_K})$ and $b_{\mathrm{res}}$ are the label and residual baselines respectively (we use zero baselines unless otherwise stated).

We define a straight-line interpolation path in the joint space:
\[
(y(\alpha), u(\alpha)) = (b_y, b_{\mathrm{res}}) + \alpha\big((y,u) - (b_y, b_{\mathrm{res}})\big), \quad \alpha \in [0,1],
\]
and the corresponding latent path
\[
z(\alpha) = (\phi_1(y_1(\alpha)), \ldots, \phi_K(y_K(\alpha)), u(\alpha)).
\]

The label-level attribution for factor $y_a$ is defined via the chain rule through the embedding $z_a=\phi_a(y_a)$ as
\begin{equation}
    \mathrm{IG}_{\ell,y_a}(z;b)
    = (y_a - b_{y_a}) \odot \int_0^1 
    \nabla_{y_a} g_\ell\big(z(\alpha)\big)\, d\alpha.
\end{equation}

The residual attribution is computed with respect to $u$ along the same path:
\begin{equation}
    \mathrm{IG}_{\ell,\mathrm{res}}(z;b)
    = (u - b_{\mathrm{res}})\odot \int_0^1 
    \nabla_u g_\ell\big(z(\alpha)\big)\, d\alpha.
\end{equation}

For $S=32$ randomly selected source samples per fold, held fixed across the $\lambda$ sweep within that fold, we average signed attributions before taking magnitudes:
\begin{equation}
 A_{\ell,y_a}=\frac{1}{q_a}\left\|\frac{1}{S}\sum_{s=1}^{S}\mathrm{IG}_{\ell,y_a}(z^{(s)};b)\right\|_1,\qquad
 A_{\ell,\mathrm{res}}=\frac{1}{m}\left\|\frac{1}{S}\sum_{s=1}^{S}\mathrm{IG}_{\ell,\mathrm{res}}(z^{(s)};b)\right\|_1,
\end{equation}
where $q_a$ is the number of input dimensions of factor $y_a$ and $m$ is the residual dimension. We use Captum's Integrated Gradients with 50 Gauss--Legendre integration steps and zero baselines.

For attribution visualizations in Fig.~\ref{fig:rib_heatmap} and Fig.~\ref{fig:allen_heatmap}, and for the attribution summaries in Fig.~\ref{fig:mcmaze}c,d, we min--max normalize the dataset-level attribution scores for each factor across neurons to the range \([0,1]\).

\[
A^{\mathrm{norm}}_{\ell,y_a}(z)
= \frac{A_{\ell,y_a}(z) - \min_{i} A_{i,y_a}(z)}
       {\max_{i} A_{i,y_a}(z) - \min_{i} A_{i,y_a}(z)}
\]
\subsection{Counterfactual Semantics}\label{cf_semantics}
The edited sample has a causal interpretation only under additional assumptions. Suppose the data are generated from observed factors \(Y=(Y_1,\ldots,Y_K)\)
and a source-specific unobserved factor \(H\), and suppose an intervention on
\(Y_a\) leaves \((Y_{-a},H)\) fixed. If the true conditional spike distribution factorizes as
\begin{equation}
    p^\star(x\mid y,h)=\prod_{\ell=1}^n \mathrm{Poisson}\left(x_\ell\mid G_\ell^\star(\phi_1^\star(y_1),\ldots,\phi_K^\star(y_K),\psi^\star(h))\right),
\end{equation}
if Neural-DISCO recovers these maps, and if the residual encoder recovers $\psi^\star(h)$ for the source sample, then replacing $\phi_a(y_a)$ by $\phi_a(\hat y_a)$ while holding $(z_{-a},u)$ fixed produces samples from
\begin{equation}
    p^\star\big(x\mid \mathrm{do}(Y_a=\hat y_a),Y_{-a}=y_{-a},H=h\big).
\end{equation}
The learned residual \(u=\phi_{\mathrm{res}}(x)\) is intended to encode
source-specific information associated with \(h\), but it is not the same object
as the latent generative factor \(H\).
When these assumptions are not justified, Neural-DISCO counterfactuals are model-based predictions under the learned factorization. Edits far outside the training support, or edits that violate natural dependencies among labels and internal state, should be treated as extrapolations.

\subsection{Hyperparameter selection and reporting}
We selected dataset-specific hyperparameters using validation diagnostics and preliminary pilot runs, then fixed one configuration for each dataset--model pair before running the cross-validation experiments reported in the main text.  The final configurations are reported in Table~\ref{tab:hyperparams}. The $\lambda$-selection criterion favors high validation edit fidelity and low supplied-label decodability from the clean residual, averaged across ten validation folds. Reconstruction loss provides an additional diagnostic; checkpoint selection uses validation loss.

The unregularized condition is an architecture-matched ablation in which the residual bottleneck is disabled by setting both the residual activation penalty and residual noise to zero. This ablation tests the necessity of the bottleneck mechanism and was not used to choose the reported Neural-DISCO hyperparameters.

We also report sweeps over the residual regularization coefficient $\lambda$ to evaluate how bottleneck strength affects attribution and residual leakage, as in Fig.~\ref{fig:allen}, Fig.~\ref{fig:mcmaze}, and Fig.~\ref{fig:rib}. These sweeps are sensitivity analyses rather than hyperparameter-selection curves for the reported cross-validation results. Initialization and training use seed $314159+\text{fold index}$, held fixed across $\lambda$. Fold indices run from 0 to 9. The sweeps retain residual noise $\sigma=.02$ for RatInABox and Allen and $.04$ for MC Maze, including at $\lambda=0$; the zero-penalty point therefore differs from the $\lambda=\sigma=0$ ablation.

\subsection{Evaluation metrics}
Unless otherwise stated, reported values are calculated on the test set, and averaged across the 10 CV-fold models. Uncertainty is reported as the standard error of the mean across folds. If a metric $m_k$ is computed on fold-models $k\in\{1,\dots,10\}$, we report with $K=10$:
\begin{equation}
\bar m = \frac{1}{K}\sum_{k=1}^{K} m_k,
\qquad
\mathrm{SEM}(m) = \frac{s_m}{\sqrt{K}},
\end{equation}
where $s_m$ is the sample standard deviation of the fold-wise values.

\paragraph{RatInABox.}
For attribution, we treat the neuron-level attribution score for a given factor as a classifier score for the corresponding synthetic cell type and report area under the ROC curve (AUC). Writing $s_i$ for the attribution score of neuron $i$ and $y_i\in\{0,1\}$ for whether that neuron belongs to the relevant synthetic cell class, the AUC is
\begin{equation}
\mathrm{AUC} = \Pr(s^+ > s^-) + \tfrac{1}{2}\Pr(s^+ = s^-),
\end{equation}
where $s^+$ and $s^-$ denote scores sampled from the positive and negative classes, respectively. Position attribution is evaluated against place-cell labels, head-direction attribution against head-direction-cell labels, and residual attribution against speed-cell labels because speed is withheld from the observed-factor embeddings. We use scikit-learn's sklearn.metrics.roc\_auc\_score to compute AUC values empirically. For counterfactual editing, we compare generated rates \(\hat r_{\ell t}\) with simulator ground-truth edited target rates \(r^{\star}_{\ell t}\). For a neuron group \(\mathcal I\), we report
\begin{equation}
R^2_{\mathcal I}
= \frac{1}{|\mathcal I|}\sum_{\ell\in\mathcal I}
\left[1 - \frac{\sum_{t=1}^{T}(\hat r_{\ell t}-r^{\star}_{\ell t})^2}
{\sum_{t=1}^{T}(r^{\star}_{\ell t}-\bar r^{\star}_{\ell})^2}\right],
\end{equation}
where
\[
\bar r^{\star}_{\ell}=\frac{1}{T}\sum_{t=1}^{T}r^{\star}_{\ell t}.
\] We also check qualitatively that edited-factor responses move toward the target while non-target and withheld components are preserved.

\paragraph{MC Maze.}
Because there is no ground-truth counterfactual neural response in the real recordings, we evaluate generated activity through downstream behavioral decoding. A CEBRA decoder trained only on real training-set spikes predicts the target hand position, hand velocity, and eye position from source spikes, target spikes, and generated activity. For a decoded trajectory $\hat y_t$ and target trajectory $y_t^{\star}$, we report the coefficient of determination
\begin{equation}
R^2 = \frac{1}{d}\sum_{j=1}^{d}
\left[1-\frac{\sum_{t=1}^{T}(\hat y_{tj}-y_{tj}^{\star})^2}
{\sum_{t=1}^{T}(y_{tj}^{\star}-\bar y_j^{\star})^2}\right],
\end{equation}
where $d=2$ is the number of coordinates and $\bar y_j^{\star}=\frac{1}{T}\sum_{t=1}^{T} y_{tj}^{\star}$. Thus, $R^2$ is averaged over coordinates. Source-spike decoding serves as a negative control and target-spike decoding serves as an empirical reference for the decoder on real data. For attribution, we again summarize performance by the AUC for distinguishing recorded units from synthetic constant-rate control neurons using the corresponding factor-attribution scores.

\paragraph{Allen Visual Coding.}
For running-speed edits, a decoder trained on real training-set activity predicts running speed from generated responses, and we report root mean squared error (RMSE) relative to the edited target speed $v_t^{\star}$:
\begin{equation}
\mathrm{RMSE} = \sqrt{
\frac{1}{T}
\sum_{t=1}^{T}
(\hat v_t-v_t^\star)^2
}.
\end{equation}
To test preservation of withheld stimulus information, we train a logistic-regression decoder on real activity to predict temporal frequency and report balanced accuracy on edited responses,
\begin{equation}
\mathrm{BA} = \frac{1}{C}\sum_{c=1}^{C} \frac{\mathrm{TP}_c}{\mathrm{TP}_c + \mathrm{FN}_c},
\end{equation}
where $C$ is the number of temporal-frequency classes. For attribution, we use two complementary summaries: AUC for classifying units with significant running modulation from running-attribution scores, and the Pearson correlation between orientation-attribution scores $a_i$ and the dataset-provided global orientation selectivity indices $g_i$,
\begin{equation}
\rho(a,g) = \frac{\sum_i (a_i-\bar a)(g_i-\bar g)}{\sqrt{\sum_i (a_i-\bar a)^2}\sqrt{\sum_i (g_i-\bar g)^2}}.
\end{equation}

\subsection{Edit--preservation score}
\label{app:edit_preservation_score}
Adapting harmonic-mean summaries used in model-editing evaluation \parencite{meng2022rome,meng2023memit}, we summarize edit fidelity and preservation as
\begin{equation}
S_{\mathrm{EP}} = \frac{2EP}{E+P},
\qquad [x]_0^1 = \min\{1,\max\{0,x\}\},
\end{equation}
where $[x]_0^1$ clips $x$ to $[0,1]$, and $S_{\mathrm{EP}}=0$ when $E+P=0$. Both components lie in $[0,1]$; a high score requires both a successful edit and preservation of the remaining information. We retain the underlying measurements in Tables~\ref{tab:rat_edit_preservation} and~\ref{tab:allen_edit_preservation} so that failures of either component remain visible.

For the RatInABox edits, speed is withheld from the observed-factor embeddings. $E=[R^2_{\mathrm{edited}}]_0^1$ evaluates the edited cell population against the exact simulator target. Preservation is the mean of the clipped $R^2$ scores for the two other populations,
\begin{equation}
P = \tfrac{1}{2}\left([R^2_{\mathrm{other}}]_0^1 + [R^2_{\mathrm{speed}}]_0^1\right).
\end{equation}
Thus, the head-direction and speed populations enter preservation for a position edit. In the head-direction edit, head-direction cells define $E$, while place and speed cells define $P$; all components use the edit's simulator target rates. For the Allen speed-to-mean edit,
\begin{equation}
E = \left[1-\frac{\mathrm{RMSE}_{\mathrm{edited}}}{\mathrm{RMSE}_{\mathrm{source}}}\right]_0^1,
\qquad
P = \left[\frac{\mathrm{BA}_{\mathrm{edited}}-0.2}{\mathrm{BA}_{\mathrm{source}}-0.2}\right]_0^1,
\end{equation}
where both RMSE values are relative to the target speed and $\mathrm{BA}$ is temporal-frequency balanced accuracy. Chance accuracy is $0.2$ for the five-way decoder; we set $P=0$ if source accuracy is at or below chance. We do not assign $S_{\mathrm{EP}}$ to the original MC Maze reach edit because all three measured factors are edited together, leaving no independently measured non-edited factor to define preservation.

The Gaussian GLM editor adds its fitted source-to-target label effect to the recorded response, $x_{\mathrm{edit}}=x_{\mathrm{source}}+W(y_{\mathrm{edit}}-y_{\mathrm{source}})$, with negative outputs clipped to zero. The Poisson GLM instead predicts rates from the target labels, $\hat r(y)=\exp(Wy+b)$. In RatInABox, both GLMs use raw position and heading components $(x,y,h_x,h_y)$ as covariates. The zero residual control retains the trained Neural-DISCO model and sets its source residual to zero at edit time. Because the Gaussian GLM edits recorded spike counts, its output keeps spike-count noise, whereas the other methods predict mean rates; all are scored against simulator mean rates.

\begin{table}[t]
\centering
\small
\setlength{\tabcolsep}{4pt}
\caption{RatInABox editing against exact simulator target trajectories. Population $R^2$ averages neuron-wise scores. All three populations are evaluated against simulator target rates in both edits. Values and joint scores are computed per fold and reported as mean $\pm$ SEM across ten folds. Each component is clipped before computing the edit--preservation score.}
\label{tab:rat_edit_preservation}
\resizebox{\textwidth}{!}{%
\begin{tabular}{lrrrr}
\toprule
Method & Edited $R^2$ & Other observed $R^2$ & Speed $R^2$ & $S_{\mathrm{EP}}$ \\
\midrule
\multicolumn{5}{l}{\textit{Position shift}} \\
Neural-DISCO & $0.940\pm0.004$ & $0.888\pm0.001$ & $0.964\pm0.001$ & $0.933\pm0.002$ \\
Zero residual & $0.976\pm0.002$ & $0.911\pm0.0002$ & $-0.004\pm0.001$ & $0.621\pm0.0004$ \\
Gaussian GLM edit & $-208.575\pm0.391$ & $-2.116\pm0.001$ & $-1.258\pm0.002$ & $0.000\pm0.000$ \\
Poisson GLM & $0.036\pm0.012$ & $0.914\pm0.00002$ & $-0.034\pm0.002$ & $0.065\pm0.019$ \\
\midrule
\multicolumn{5}{l}{\textit{Head-direction reversal}} \\
Neural-DISCO & $0.885\pm0.001$ & $0.966\pm0.002$ & $0.794\pm0.001$ & $0.883\pm0.001$ \\
Zero residual & $0.911\pm0.0002$ & $0.986\pm0.001$ & $-0.028\pm0.001$ & $0.640\pm0.001$ \\
Gaussian GLM edit & $-1.513\pm0.0002$ & $-2.436\pm0.001$ & $-1.339\pm0.002$ & $0.000\pm0.000$ \\
Poisson GLM & $0.914\pm0.00002$ & $0.516\pm0.001$ & $-0.032\pm0.002$ & $0.402\pm0.0006$ \\
\bottomrule
\end{tabular}}
\end{table}

We analyze 32 Allen Brain Observatory sessions, retaining units with $\mathrm{SNR}>1$ and an ISI-violation score $<0.5$. Within each session, we use ten folds and edit running speed toward its mean in the fold's training data. All methods share speed and temporal-frequency decoders trained on real activity in the same training fold. Both GLMs are trained for 100 epochs with learning rate $10^{-3}$, weight decay $10^{-5}$, and seed $314159+\text{fold index}$.

\begin{table}[t]
\centering
\small
\setlength{\tabcolsep}{6pt}
\caption{Allen speed editing and temporal-frequency preservation across 32 sessions, using a common fold-training mean target. Metrics and $S_{\mathrm{EP}}$ are computed per fold and averaged over ten folds within each session; values are mean $\pm$ SEM across the 32 session means. Lower speed RMSE and higher TF balanced accuracy and $S_{\mathrm{EP}}$ are better.}
\label{tab:allen_edit_preservation}
\begin{tabular}{lrrr}
\toprule
Method & Speed RMSE (cm/s) & TF balanced accuracy & $S_{\mathrm{EP}}$ \\
\midrule
Neural-DISCO & $3.629\pm0.567$ & $0.541\pm0.013$ & $0.592\pm0.026$ \\
w/o reg. & $6.182\pm1.067$ & $0.570\pm0.013$ & $0.391\pm0.030$ \\
Zero residual & $3.114\pm0.462$ & $0.199\pm0.0004$ & $0.012\pm0.001$ \\
pi-VAE & $4.004\pm0.684$ & $0.197\pm0.0008$ & $0.010\pm0.001$ \\
NEDS & $3.979\pm0.645$ & $0.194\pm0.0008$ & $0.007\pm0.001$ \\
Gaussian GLM edit & $6.017\pm1.044$ & $0.597\pm0.012$ & $0.372\pm0.040$ \\
Poisson GLM & $5.444\pm1.042$ & $0.198\pm0.0007$ & $0.009\pm0.001$ \\
\bottomrule
\end{tabular}
\end{table}

Neural-DISCO has the highest joint score for both RatInABox edits. Zeroing its residual maintains edit fidelity but loses withheld speed information. Across the 32 Allen sessions, Neural-DISCO has a higher mean joint score than Gaussian GLM and exceeds it in 30 of 32 sessions. The Gaussian editor has higher mean temporal-frequency balanced accuracy, whereas Neural-DISCO has lower mean speed RMSE; zeroing the residual reduces temporal-frequency decoding to approximately chance. These results distinguish fitting a requested target from preserving source information.

\subsection{Editing toward inconsistent behavioral targets}
\label{app:implausible_targets}
To test whether decoding of generated activity tracks requested labels regardless of their consistency, we retain the source residual and the requested target hand-position and eye-position trajectories, but rotate the target two-dimensional hand velocity. Rotations preserve every velocity magnitude while making velocity increasingly inconsistent with the unchanged hand-position trajectory. The model is evaluated without training on the rotated targets. Within each fold, we apply a fixed CEBRA decoder trained only on real activity to Poisson spike samples from the generated rates. We average clockwise and counterclockwise results within each fold at intermediate angles, then report mean $\pm$ SEM across ten folds.

\begin{table}[t]
\centering
\small
\caption{MC Maze decoding under velocity rotations. Position and eye targets remain unchanged.}
\label{tab:implausible_targets}
\begin{tabular}{rrrr}
\toprule
Velocity rotation & Velocity $R^2$ & Hand-position $R^2$ & Eye-position $R^2$ \\
\midrule
$0^\circ$ & $0.772\pm0.003$ & $0.895\pm0.002$ & $0.630\pm0.004$ \\
$45^\circ$ & $0.618\pm0.002$ & $0.845\pm0.002$ & $0.578\pm0.003$ \\
$90^\circ$ & $0.344\pm0.006$ & $0.744\pm0.004$ & $0.484\pm0.005$ \\
$135^\circ$ & $0.160\pm0.011$ & $0.647\pm0.005$ & $0.410\pm0.004$ \\
$180^\circ$ & $0.102\pm0.009$ & $0.603\pm0.008$ & $0.388\pm0.008$ \\
\bottomrule
\end{tabular}
\end{table}

Requested-velocity decoding declines progressively as the supplied velocity becomes less compatible with the fixed position trajectory. Hand and eye decoding also decline, although less strongly. The control therefore shows degradation under inconsistent behavioral requests with collateral effects on the unchanged targets, rather than selective rejection of velocity alone. This is evidence against indiscriminate decoder readout of supplied labels; it does not establish biological validity of the generated responses or causal effects of the edits.

\subsection{Orientation and temporal-frequency editing}
\label{app:allen_orientation}
For orientation edits, the observed factors are grating orientation and temporal frequency (TF), and running speed is withheld. We rotate orientation by $180^\circ$ across the same 32 sessions used for speed editing. Models and evaluation decoders are trained on the training presentations in each fold, with model selection on separate validation presentations. Edits are evaluated on the final drifting-grating epoch. Neural-DISCO retains its source residual; the zero residual control sets this residual to zero at edit time. The unregularized ablation, pi-VAE and NEDS use the same orientation and TF inputs. Training settings are given in Table~\ref{tab:hyperparams}.

We compute edit fidelity $E$ as the mean chance-corrected balanced accuracy over edited stimulus factors. Chance levels are $1/8$ for orientation and $1/5$ for TF. Preservation $P$ averages the chance-corrected accuracy retained relative to source decoding for each unedited stimulus factor and the running-speed component $(\mathrm{RMSE}_0-\mathrm{RMSE}_{\mathrm{edited}})/(\mathrm{RMSE}_0-\mathrm{RMSE}_{\mathrm{source}})$, where the RMSEs are speed-decoding errors on edited and source activity and $\mathrm{RMSE}_0$ is the error of always predicting the mean held-out speed. The speed component is zero when source decoding is no more accurate than this constant prediction. Each component is clipped to $[0,1]$ before averaging, and $S_{\mathrm{EP}}=2EP/(E+P)$. When both stimulus factors are edited, preservation measures running speed alone.

\begin{table}[t]
\centering
\small
\setlength{\tabcolsep}{4pt}
\caption{Orientation editing and source preservation across 32 Allen sessions. Metrics and $S_{\mathrm{EP}}$ are computed per fold and averaged over ten folds within each session; values are mean $\pm$ SEM across the 32 session means. All methods share decoders trained on real fold-training activity and are evaluated on the same held-out presentations and targets. Speed RMSE measures error relative to source speed. Lower speed RMSE and higher balanced accuracy and $S_{\mathrm{EP}}$ are better.}
\label{tab:allen_orientation_components}
\begin{tabular}{lrrrr}
\toprule
Method & Orientation BA & TF BA & Speed RMSE (cm/s) & $S_{\mathrm{EP}}$ \\
\midrule
Neural-DISCO & $0.949\pm0.006$ & $0.823\pm0.012$ & $6.268\pm1.017$ & $0.838\pm0.022$ \\
w/o reg. & $0.842\pm0.012$ & $0.693\pm0.013$ & $5.613\pm0.933$ & $0.807\pm0.024$ \\
Zero residual & $0.961\pm0.005$ & $0.855\pm0.013$ & $11.943\pm1.741$ & $0.656\pm0.002$ \\
Gaussian GLM & $0.714\pm0.016$ & $0.532\pm0.015$ & $6.119\pm1.015$ & $0.705\pm0.023$ \\
Poisson GLM & $0.976\pm0.003$ & $0.914\pm0.008$ & $11.894\pm1.751$ & $0.660\pm0.001$ \\
pi-VAE & $0.317\pm0.015$ & $0.620\pm0.018$ & $11.743\pm1.720$ & $0.263\pm0.013$ \\
NEDS & $0.601\pm0.030$ & $0.709\pm0.020$ & $11.958\pm1.763$ & $0.399\pm0.021$ \\
\bottomrule
\end{tabular}
\end{table}

Neural-DISCO has higher orientation fidelity and TF preservation than Gaussian GLM, pi-VAE, NEDS and the unregularized ablation. Poisson GLM and the zero residual control predict the stimulus labels accurately but do not preserve source-speed information.

\paragraph{Additional single-session stimulus edits.}
In session 756029989, we also advance TF cyclically through $1,2,4,8,15$ Hz or jointly edit orientation and TF. Table~\ref{tab:allen_additional_stimulus_edits} reports the component and joint scores for these edits.

\begin{table}[t]
\centering
\small
\setlength{\tabcolsep}{7pt}
\caption{Additional Allen stimulus edits in session 756029989. Values are mean $\pm$ SEM across ten folds, with $E$, $P$ and $S_{\mathrm{EP}}$ computed within each fold before aggregation. All methods share decoders trained on real fold-training activity and are evaluated on the same held-out presentations and targets. Higher values are better.}
\label{tab:allen_additional_stimulus_edits}
\begin{tabular}{lrrr}
\toprule
Method & Edit fidelity $E$ & Preservation $P$ & $S_{\mathrm{EP}}$ \\
\midrule
\multicolumn{4}{l}{\textit{Temporal-frequency shift}} \\
Neural-DISCO & $0.840\pm0.005$ & $0.895\pm0.010$ & $0.867\pm0.005$ \\
w/o reg. & $0.479\pm0.018$ & $0.956\pm0.009$ & $0.636\pm0.016$ \\
Gaussian GLM & $0.466\pm0.010$ & $0.919\pm0.003$ & $0.618\pm0.009$ \\
Poisson GLM & $0.939\pm0.002$ & $0.500\pm0.000$ & $0.653\pm0.0004$ \\
pi-VAE & $0.647\pm0.040$ & $0.145\pm0.040$ & $0.226\pm0.047$ \\
NEDS & $0.737\pm0.026$ & $0.350\pm0.076$ & $0.426\pm0.093$ \\
Zero residual & $0.874\pm0.007$ & $0.500\pm0.000$ & $0.636\pm0.002$ \\
\midrule
\multicolumn{4}{l}{\textit{Joint orientation and temporal-frequency edit}} \\
Neural-DISCO & $0.906\pm0.002$ & $0.789\pm0.020$ & $0.842\pm0.012$ \\
w/o reg. & $0.661\pm0.009$ & $0.918\pm0.018$ & $0.767\pm0.009$ \\
Gaussian GLM & $0.567\pm0.006$ & $0.933\pm0.004$ & $0.705\pm0.005$ \\
Poisson GLM & $0.958\pm0.001$ & $0.000\pm0.000$ & $0.000\pm0.000$ \\
pi-VAE & $0.439\pm0.055$ & $0.000\pm0.000$ & $0.000\pm0.000$ \\
NEDS & $0.697\pm0.082$ & $0.000\pm0.000$ & $0.000\pm0.000$ \\
Zero residual & $0.927\pm0.004$ & $0.000\pm0.000$ & $0.000\pm0.000$ \\
\bottomrule
\end{tabular}
\end{table}

The component scores distinguish successful stimulus editing from preservation of source running speed. Poisson GLM attains high edit fidelity but has a zero preservation score when orientation and TF are edited together. It generates rates from the target stimulus labels without a source residual. pi-VAE and NEDS also have zero preservation scores for this joint edit. Gaussian GLM has higher preservation than Neural-DISCO for the joint edit, while Neural-DISCO has higher edit fidelity and joint score. Zeroing the residual yields a zero joint-edit preservation score.

\subsection{Decoding omitted and supplied factors from the clean residual}
\label{app:allen_clean_residual}
We trained three Allen models, each supplied with two factors while withholding the third, using $\lambda=.01$. We extracted the clean residual without injecting noise and trained probes only on residuals from the model's training fold. Categorical probes used a training-fitted standardizer followed by balanced logistic regression; the speed probe used a training-fitted standardizer and ridge regression. Probes were evaluated on the corresponding held-out validation fold.

\begin{table}[t]
\centering
\small
\setlength{\tabcolsep}{6pt}
\caption{Decoding from the clean Allen residual. Each row represents a separately trained model; columns identify the factor decoded from its residual. Bold diagonal entries measure retention of the omitted factor, while off-diagonal entries measure decodability of supplied factors. Values are mean $\pm$ SEM across ten folds. Balanced accuracies are reported without chance correction; speed uses raw $R^2$.}
\label{tab:allen_clean_residual}
\begin{tabular}{lrrr}
\toprule
Factor withheld & Orientation BA & Speed $R^2$ & TF BA \\
\midrule
Orientation & $\mathbf{.486\pm.007}$ & $.201\pm.084$ & $.252\pm.016$ \\
Speed & $.142\pm.010$ & $\mathbf{.649\pm.034}$ & $.265\pm.013$ \\
Temporal frequency & $.163\pm.008$ & $.083\pm.104$ & $\mathbf{.550\pm.012}$ \\
\midrule
Chance / reference & $.125$ & $0$ & $.20$ \\
\bottomrule
\end{tabular}
\end{table}

The residual preferentially retains omitted-factor information, although some supplied-factor information remains. Each factor is more decodable when omitted from the observed inputs; supplied-factor decoding includes TF balanced accuracies of $.252$ and $.265$ and speed $R^2=.201\pm.084$.

\paragraph{MC Maze residual decoding.}
We decode supplied hand position, hand velocity, and eye position from the clean residual using MLP decoders on held-out data. Table~\ref{tab:mc_clean_residual} also includes source-spike decoding as a reference, using Ridge probes ($\alpha=1$). Feature standardization and training of these source-spike probes use only the training trials of each cross-validation fold; evaluation uses held-out validation trials. The clean residual is the residual encoder output without
injected noise. Table~\ref{tab:mc_clean_residual} shows low residual decoding
relative to source spikes. These decoding scores do not establish statistical independence.

\begin{table}[t]
  \centering
  \caption{Held-out decoding of supplied MC Maze covariates from the clean residual, with source-spike decoding as a reference. Values are mean $\pm$ SEM across ten folds. Source-spike decoding uses Ridge probes and uniform-average $R^2$ across the two coordinates. Negative $R^2$ is retained.}
  \label{tab:mc_clean_residual}
  \begin{tabular}{lcc}
    \toprule
    Covariate & Clean residual $R^2$ & Source spikes $R^2$ \\
    \midrule
    Hand position & $0.045 \pm 0.003$ & $0.448 \pm 0.004$ \\
    Hand velocity & $0.049 \pm 0.005$ & $0.387 \pm 0.005$ \\
    Eye position & $0.053 \pm 0.003$ & $0.364 \pm 0.007$ \\
    \bottomrule
  \end{tabular}
\end{table}

\subsection{Robustness to mixed selectivity}
\label{app:mixed_selectivity}
Motivated by mixed selectivity in navigational coding \parencite{hardcastle2017multiplexed}, we tested dependence on cleanly separated tuning by introducing multiplicative interactions into the 300-neuron RatInABox population. The three 100-neuron groups combined place $\times$ speed, head direction $\times$ position, and speed $\times$ head direction tuning. For a neuron with original rate $r_i^a(t)$ and a fixed paired neuron from another population with rate $r_{\pi(i)}^b(t)$, we used
\begin{equation}
\widetilde r_i(t)
= c_i r_i^a(t)\left[(1-\alpha)+\alpha\frac{r_{\pi(i)}^b(t)}{\overline{r_{\pi(i)}^b}}\right],
\end{equation}
where the overbar denotes the mean on the original training trajectory and $c_i$ preserves the neuron's mean training rate. We retrained Neural-DISCO at each interaction strength using the RatInABox architecture and training settings in Table~\ref{tab:hyperparams}. The simulator supplied exact counterfactual rates by combining the requested target factor with the preserved source factors.

Table~\ref{tab:mixed_selectivity} reports the mean of neuron-wise $R^2$ across all 300 neurons, followed by the edit--preservation score defined in Appendix~\ref{app:edit_preservation_score}. The $\alpha=0$ row comes from separately trained models scored against target rates constructed for this study, so values differ slightly from Table~\ref{tab:rat_edit_preservation}. For this experiment, edit fidelity uses the primary edited population, while preservation averages the other two populations' $R^2$ scores after clipping each to $[0,1]$. All components are scored against the exact edited simulator rates: neurons recruited through an interaction may change even when their primary factor is unchanged. At $\alpha=1$, paired all-population $R^2$ retained $99.17\pm0.33\%$ of its $\alpha=0$ value for position edits and $99.29\pm0.24\%$ for direction edits. These results support robustness to the simulated interactions tested here.

\begin{table}[htbp]
\centering
\small
\setlength{\tabcolsep}{5pt}
\begin{tabular}{ccccc}
\toprule
& \multicolumn{2}{c}{Position edit} & \multicolumn{2}{c}{Direction edit} \\
\cmidrule(lr){2-3}\cmidrule(lr){4-5}
$\alpha$ & All-population $R^2$ & $S_{\mathrm{EP}}$ & All-population $R^2$ & $S_{\mathrm{EP}}$ \\
\midrule
0 & $0.937\pm.002$ & $0.941\pm.003$ & $0.945\pm.001$ & $0.931\pm.001$ \\
0.25 & $0.930\pm.002$ & $0.934\pm.003$ & $0.939\pm.001$ & $0.927\pm.001$ \\
0.50 & $0.931\pm.001$ & $0.937\pm.001$ & $0.940\pm.001$ & $0.933\pm.001$ \\
0.75 & $0.930\pm.001$ & $0.936\pm.002$ & $0.941\pm.001$ & $0.936\pm.001$ \\
1 & $0.929\pm.001$ & $0.933\pm.001$ & $0.938\pm.001$ & $0.937\pm.001$ \\
\bottomrule
\end{tabular}
\caption{Neural-DISCO counterfactual performance under increasing multiplicative mixed selectivity. Values are mean $\pm$ SEM across ten independently trained cross-validation folds. The $R^2$ columns average neuron-wise scores across the full population; $S_{\mathrm{EP}}$ combines primary-population edit fidelity with performance on the other two populations.}
\label{tab:mixed_selectivity}
\end{table}

\subsection{Data and sampling requirements}
\label{app:data_requirements}
We retrained Neural-DISCO on nested, fold-specific subsets of the training data, holding validation and test sets, counterfactual edits, and training settings fixed within each dataset. Reduced RatInABox subsets were contiguous trajectory intervals that did not cross the held-out validation interval. MC Maze subsets contained complete trials balanced across 108 task conditions; Allen subsets contained complete presentations balanced across 40 orientation--temporal-frequency conditions. Evaluation decoders were held fixed across fractions: MC Maze used the full real training recording, and Allen used the complete real training fold. The Allen scaling experiment edited speed to the mean of the full training fold, held fixed across generator-training fractions.

We computed retention relative to each fold's full-data retraining and then averaged across ten folds. Retention is the reduced-data/full-data ratio for $R^2$, the ratio after subtracting the $0.20$ chance level for Allen temporal-frequency balanced accuracy, and the full-data/reduced-data ratio for Allen speed RMSE. Table~\ref{tab:data_requirements} gives the smallest tested subset whose mean paired retention reached $90\%$ for every evaluated outcome. This criterion applies to the fold mean and does not require each individual fold to exceed $90\%$.

\begin{table}[htbp]
\centering
\small
\setlength{\tabcolsep}{3pt}
\begin{tabular}{>{\raggedright\arraybackslash}p{.15\linewidth}>{\raggedright\arraybackslash}p{.23\linewidth}>{\raggedright\arraybackslash}p{.18\linewidth}>{\raggedright\arraybackslash}p{.17\linewidth}r}
\toprule
Dataset & Training subset & Condition coverage & Limiting outcome & Retention \\
\midrule
RatInABox & 3.6 min; 4,320 bins (25\%) & Contiguous interval & Position-edit $R^2$ & $95.1\pm1.9\%$ \\
MC Maze & 145 trials (9.375\%) & 108/108; median 1 trial & Velocity $R^2$ & $93.9\pm.3\%$ \\
Allen & 170 presentations (50\%) & 40/40; median 4 presentations & TF balanced accuracy & $96.3\pm2.4\%$ \\
\bottomrule
\end{tabular}
\caption{Tested training-subset requirements for retaining at least $90\%$ of full-data performance on every outcome. Retention is mean $\pm$ SEM across ten folds for the limiting outcome; the full-fold reference equals $100\%$ by definition. Median counts are per condition. These training-subset sizes assume the same held-out data and full-data evaluation decoders; the Allen speed target uses the full training fold.}
\label{tab:data_requirements}
\end{table}

For Allen, 170 presentations (50\%) met both retention criteria. At the next smaller tested subset (37.5\%; 128 presentations), TF retention was $88.4\pm3.0\%$. Retention is not clipped at $100\%$. At the next smaller tested subsets, the limiting retention fell to $78.9\pm3.0\%$ for RatInABox at 1.8 minutes and $89.4\pm.8\%$ for MC Maze at 97 trials. RatInABox tested five training fractions; MC Maze and Allen each tested nine. These results apply to the tested training subsets and do not establish universal data requirements.

Training used seed $314159+\text{fold index}$ and subset selection used seed $271828+\text{fold index}$, with 300 epochs for RatInABox and 100 epochs for each real dataset.

\subsection{Hypothesis generation from source-dependent tuning}
\label{app:hypothesis_generation}
We use Allen Visual Coding recordings to test whether a score constructed without temporal-frequency (TF) labels can prioritize neurons whose activity carries TF information in a later epoch. The models are trained with running speed and orientation; TF is excluded from Neural-DISCO training, donor matching, and score construction. Scoring uses out-of-fold responses from the first two drifting-grating epochs; its association with TF contribution is evaluated in the final epoch. Changing only the residual shows how much a neuron's predicted tuning varies with source context. The score provides a ranking for follow-up experiments, which we evaluate against TF contribution in a later epoch.

For each source bin, we choose a donor from a different presentation in the same validation fold, matching presented orientation and within-presentation time bin, with the nearest running speed. We generate two predicted tuning curves over the same eight target orientations using the source running speed in both cases: one uses the source residual and the other the donor residual. Thus, the observed inputs to the two predicted curves are identical, and only the residual differs. Two of 378 early presentations lack a valid donor and are excluded, leaving 376 presentations. For each presentation, predicted rates are averaged across its bins before calculating the squared log-rate difference between curves. The neuron score averages this difference over orientations and presentations:
\begin{equation}
    c_\ell = \frac{1}{376}\sum_{p=1}^{376}\frac{1}{8}\sum_{o=1}^{8}
    \left[\log(\bar r^{\mathrm{source}}_{\ell po}+\epsilon)
    -\log(\bar r^{\mathrm{donor}}_{\ell po}+\epsilon)\right]^2,
    \qquad \epsilon=10^{-4}.
\end{equation}

Independent TF contribution is evaluated on 222 final-epoch presentations using presentation-grouped two-fold cross-validation, repeated five times. The reduced Poisson model contains categorical orientation and linear running speed; the full model additionally contains categorical TF, and the null model contains only an intercept. Each neuron's contribution is $(D_{\mathrm{reduced}}-D_{\mathrm{full}})/D_{\mathrm{null}}$ \parencite{tseng2022shared}, where the deviances are evaluated on held-out presentations. Running-speed scaling and the null rate are fitted only within each training half and then applied to its held-out counterpart.

We compare Neural-DISCO with an out-of-fold, identity-link Gaussian GLM with categorical orientation and linear running speed. Its edit adds the GLM's predicted change to the source rates: $r_{\mathrm{edit}}=r_{\mathrm{source}}-f(y_{\mathrm{source}})+f(y_{\mathrm{target}})$. Using identical source--donor pairs, its score is the squared difference between presentation-averaged raw-rate curves, averaged over orientations and presentations. The shared target prediction cancels in this difference, leaving the paired additive residual difference. We also evaluate squared log-rate differences after flooring edited rates at $\epsilon$ before presentation averaging.

\begin{table}[t]
\centering
\caption{Hypothesis prioritization across all 672 units and the anatomical subset of 311 visual-cortex units. Spearman correlation between the early-epoch score and final-epoch TF contribution. Values are correlation $\pm$ bootstrap standard error from 2,000 presentation resamples, stratified by fold and orientation in the early epochs and by orientation and TF in the final epoch.}
\label{tab:hypothesis_generation}
\begin{tabular}{lrr}
\toprule
Editor & All units & Visual cortex \\
\midrule
Neural-DISCO & $0.386 \pm 0.030$ & $0.448 \pm 0.030$ \\
Gaussian GLM (raw-rate score) & $0.192 \pm 0.022$ & $0.186 \pm 0.027$ \\
Gaussian GLM (log-rate score) & $-0.061 \pm 0.020$ & $0.121 \pm 0.024$ \\
\bottomrule
\end{tabular}
\end{table}

The visual-cortex comparison in Table~\ref{tab:hypothesis_generation} supports using source-dependent predicted tuning to prioritize hypotheses about an omitted experimental variable. It does not identify a unique hidden cause or establish a causal effect of residual state.

\subsection{Compute resources}
\label{compute}
All experiments were run on a single H200 GPU with 141 GB RAM. Training and evaluation of Neural-DISCO and all baselines across all datasets required approximately 6 GPU hours in total. RatInABox sweeps over $\lambda$ required approximately 2 GPU hours, MC Maze $\lambda$ sweeps required approximately 1 GPU hour, Allen $\lambda$ sweeps required approximately 1 GPU hour, and baseline training and evaluation required approximately 2 GPU hours. During Neural-DISCO training, peak VRAM usage stayed below 1 GB.

\subsection{RatInABox synthetic dataset experiment details}
\paragraph{Dataset generation.}
We use the RatInABox generator~\parencite{ratinabox} to generate synthetic neural spiking activity, simulator firing rates, and associated behavioral variables. The synthetic environment has aspect ratio 2:1 and scale 1 m. Dataset generation uses random seed 42. The simulation timestep is 50 ms, the duration is 20 min, the initial position is $(0.5,0.5)$, and the mean speed is 0.2 m/s.
We generate four trajectories:
\begin{itemize}
    \item \textbf{Train:} standard random-walk trajectory.
    \item \textbf{Inverted:} the same trajectory as \texttt{train}, but temporally reversed.
    \item \textbf{Shifted:} a trajectory spatially confined to the left side of the environment, $x\leq 1$.
    \item \textbf{Cfact\_shifted:} the same as \texttt{shifted}, but with a $+1$ offset along the $x$-axis.
\end{itemize}
\paragraph{Synthetic cell types.}
The synthetic population contains 300 neurons: 100 place cells, 100 head-direction cells, and 100 speed cells.
\paragraph{Loaded arrays and observed variables.}
Each generated RatInABox file contains 24,000 bins. Spike counts and simulator rates have shape $(24000,300)$, position has shape $(24000,2)$, head direction has shape $(24000,2)$, and velocity has shape $(24000,2)$. Neural-DISCO receives position and head direction as observed editable factors. Speed is withheld from the observed-factor embeddings, so speed-related information can only be preserved through the residual latent.
\paragraph{Splits and counterfactual edits.}
RatInABox uses contiguous sample splits rather than trial-grouped splits. Each fold contains 17,280 training samples and 1,920 validation samples. The held-out test loader uses the first 20\% of the \texttt{inverted} trajectory, corresponding to 4,800 samples.
For counterfactual generation, the model keeps the source residual and swaps the configured target behavior blocks. In the position-shift edit, the target position comes from the shifted counterfactual trajectory while head direction and speed are held fixed. In the head-direction edit, head direction is reversed while position and speed are held fixed. Since the simulator provides ground-truth rates under the edited variables, generated rates are compared directly against the edit's simulator target rates by cell group: place cells, head-direction cells, speed cells, and all cells.

\subsubsection{Generalization to unseen factor combinations}
\label{app:rat_unseen_combinations}
During training, we withhold all samples whose position satisfies $x\in[1.5,2.0]$, $y\in[0.5,1.0]$, and whose heading has a non-negative $y$-component ($h_y\ge0$). Here $h_y$ is the $y$-component of the unit head-direction vector. The joint position--heading combinations are therefore absent from training. Table~\ref{tab:rat_unseen_combinations} reports position-edit fidelity and preservation of the withheld speed factor: pi-VAE achieves slightly higher position fidelity, whereas Neural-DISCO preserves speed-cell activity.

\begin{table}[t]
\centering
\small
\caption{Position editing with joint position--heading combinations withheld from training. Position-edit and speed-preservation $R^2$ average neuron-wise scores within the place-cell and speed-cell populations, respectively, against simulator target rates. Values are mean $\pm$ SEM across ten cross-validation folds.}
\label{tab:rat_unseen_combinations}
\begin{tabular}{lcc}
\toprule
Method & Position-edit $R^2$ & Speed-preservation $R^2$ \\
\midrule
Neural-DISCO & $0.940\pm0.004$ & $0.964\pm0.001$ \\
pi-VAE & $0.966\pm0.001$ & $-0.041\pm0.003$ \\
\bottomrule
\end{tabular}
\end{table}

\subsection{MC Maze Experiment Details}
\paragraph{Dataset.}
We use the MC Maze reaching dataset from the Neural Latents Benchmark / DANDI dataset \texttt{000128/0.220113.0400}, subject \texttt{sub-Jenkins}. Spikes are loaded from \texttt{train\_spikes\_heldin}; no additional unit filtering is applied beyond the NLB held-in unit selection. The dataset contains 137 recorded units together with paired behavioral trajectories. Spike counts are binned at 50~ms, yielding 14 bins per trial in the benchmark representation. We use the NLB train split for 10-fold cross-validation and reserve the public NLB validation split as the held-out test set for decoding analyses.

\paragraph{Observed factors and preprocessing.}
The observed editable factors are hand position, hand velocity, and eye position, each represented as a two-dimensional trajectory. We do not apply additional spike smoothing beyond the benchmark binning. Behavioral targets are normalized per coordinate before training and decoding.

\paragraph{Splits.}
The NLB train split contains 1,721 trials and the public validation split contains 574 trials. After flattening 14 time bins per trial, these correspond to 24,094 train-split samples and 8,036 held-out samples, respectively. Cross-validation is grouped by full trials rather than individual bins so that all time bins from a trial remain in the same fold. In fold 0, the model split contains 21,672 training samples and 2,422 validation samples.

\paragraph{Training details.}
The MC Maze Neural-DISCO model is trained for 100 epochs with train/validation/test batch sizes of 512/256/256, learning rate $10^{-3}$, weight decay $10^{-5}$, and scheduler decay factor $\gamma=1$. The behavior latent dimension is $d_y=96$ for each observed factor and the residual latent dimension is $d_{\mathrm{res}}=128$. The factor encoder, residual encoder, and generator have depths of 3/4/4, with residual penalty $\lambda=0.48$ and residual noise standard deviation $\sigma=0.04$. The unregularized control uses the same architecture and optimizer settings with $\lambda=0$ and $\sigma=0$.

\paragraph{Source-target pairing and counterfactual construction.}
For each held-out source trial, we choose a target trial with matched task context and a different reach trajectory. Pairing uses the \texttt{context\_matched} strategy: candidate targets must match the source trial on \texttt{trial\_version}, \texttt{num\_targets}, \texttt{num\_barriers}, and \texttt{active\_target}, and must differ in target direction by at least $30^\circ$. Among the eligible candidates, the target trial is chosen by nearest standardized delay and reaction time. Pairing uses seed 0. On the public validation split, this procedure retains 557 source-target pairs from 574 trials. In the main counterfactual experiment, the hand-position, hand-velocity, and eye-position blocks are jointly replaced with the target-trial values while the source residual is held fixed.

\paragraph{Decoder training and attribution controls.}
Downstream MC Maze decoders are trained only on real training-split spikes and behavior. We use a KNN decoder on CEBRA embedding with the \texttt{offset10-model}, batch size 1000, learning rate $10^{-4}$, temperature 1.0, output dimension 8, 15{,}000 training iterations, cosine distance, \texttt{time\_delta} conditioning, time offset 10, KNN with $K=36$, cosine metric, and seed 0. For paper figures, the same trained decoder is applied to source spikes, target spikes, and generated activity. Generated activity is obtained by Poisson sampling from the counterfactual Neural-DISCO rates before downstream decoding. For attribution sanity checks, recorded neurons are augmented with synthetic constant-rate control neurons, and we report AUC for distinguishing recorded units from these controls using the total attribution score to all behaviors for each neuron.

\subsection{Allen Visual Coding Neuropixels experiment details}
\paragraph{Spike binning and preprocessing.}
For the 32-session speed- and orientation-editing comparisons, we apply the preprocessing and splitting protocol below within each session. Examples are drawn from the 32-session cohort. We use only spiking and behavior during drifting-grating stimulus presentations, filtering out gray periods in between drifting stimuli. For speed editing, running speed and stimulus orientation are used as observed factors, while drifting-grating temporal frequency is withheld from the observed-factor embeddings for preservation analyses. We remove presentations with blank or invalid orientation labels. In session 756029989, the filtered dataset contains 378 training/CV presentations and 222 held-out test presentations.
Unit filtering uses strict inequalities $\mathrm{SNR}>1$ and \texttt{isi\_violations}$<0.5$, leaving 672 units in session 756029989. Spike trains and running speed are binned at 50 ms resolution over each 2 s drifting-grating presentation, yielding 40 bins per presentation.

\paragraph{Data splitting.}
Train/test splitting is by drifting-grating epoch: the first two drifting-grating epochs are used for model training and cross-validation, and the final epoch is reserved as held-out test data. Cross-validation splits are grouped by \texttt{presentation\_id} using deterministic contiguous \texttt{np.array\_split} partitions of sorted presentation ids. In fold 0 of session 756029989, the training split contains 13,600 bins from 340 presentations, and the validation split contains 1,520 bins from 38 presentations. The held-out test split contains 8,880 bins from 222 presentations.
\paragraph{Decoder training.}
For speed-edit evaluation, we train an MLP decoder only on real fold-training activity. Real train/test spikes are Gaussian-smoothed with $\sigma=2.0$ bins before decoder fitting and scoring. The speed decoder has hidden dimension 256, 3 layers, GELU nonlinearities, dropout 0.0, learning rate $10^{-3}$, weight decay $10^{-4}$, maximum training duration 300 epochs, patience 20, and seed $161803+\text{fold index}$. Its internal validation split uses seed $42+\text{fold index}$.
For temporal-frequency preservation, we train a multinomial logistic-regression decoder only on real training activity and evaluate it on edited activity
\[
\begin{aligned}
\texttt{LogisticRegression(max\_iter=2000,} \\
\texttt{solver="lbfgs", class\_weight="balanced")}
\end{aligned}
\]
\paragraph{Unit metrics for attribution experiments.}
For the experiment shown in Fig.~\ref{fig:allen}f, ground-truth labels for running modulation are derived from the Allen unit metric \texttt{run\_pval\_dg}, with units classified as significantly modulated when \texttt{run\_pval\_dg}$<0.05$. In Fig.~\ref{fig:allen}l, we compare orientation-attribution scores with the Allen-provided global orientation selectivity index, \texttt{g\_osi\_dg}.
Stimulus examples report the most frequent decoded label within each presentation.
\paragraph{Training time.}
Training the Allen 10-fold models across the reported $\lambda$ values takes approximately 1 GPU hour on an H200 GPU.
 
\begin{table*}[t]
\centering
\footnotesize
\setlength{\tabcolsep}{3pt}
\caption{Hyperparameters and training settings for the main editing experiments. For Neural-DISCO, architecture entries give factor encoder / residual encoder / generator depths. Depths count input, hidden and output layers, so depth $d$ has $d-1$ weight matrices. Batch sizes are train/validation/test. The attribution experiments sweep $\lambda$. The unregularized control uses the same architecture and training settings as Neural-DISCO within each dataset and task, with $\lambda=0$ and $\sigma=0$.}
\label{tab:hyperparams}
\textbf{Architecture and residual bottleneck}\par\smallskip
\begin{tabular*}{\textwidth}{@{\extracolsep{\fill}}>{\raggedright\arraybackslash}p{.105\textwidth}>{\raggedright\arraybackslash}p{.135\textwidth}>{\raggedright\arraybackslash}p{.190\textwidth}>{\raggedright\arraybackslash}p{.185\textwidth}>{\raggedright\arraybackslash}p{.145\textwidth}>{\raggedleft\arraybackslash}p{.06\textwidth}>{\raggedleft\arraybackslash}p{.05\textwidth}@{}}
\toprule
Dataset & Model & Labels & Latents & Architecture & $\lambda$ & $\sigma$ \\
\midrule
RatInABox & Neural-DISCO & pos, head dir & $d_y=16$, $d_{\mathrm{res}}=32$ & 3/3/3 & 0.10 & 0.02 \\
RatInABox & unregularized & pos, head dir & $d_y=16$, $d_{\mathrm{res}}=32$ & 3/3/3 & 0 & 0 \\
MC Maze & Neural-DISCO & hand pos, hand vel, eye pos & $d_y=96$, $d_{\mathrm{res}}=128$ & 3/4/4 & 0.48 & 0.04 \\
MC Maze & unregularized & hand pos, hand vel, eye pos & $d_y=96$, $d_{\mathrm{res}}=128$ & 3/4/4 & 0 & 0 \\
Allen speed & Neural-DISCO & speed, orientation & $d_y=48$, $d_{\mathrm{res}}=96$ & 3/4/2 & 0.0015 & 0.02 \\
Allen stimulus & Neural-DISCO & orientation, TF & $d_y=48$, $d_{\mathrm{res}}=96$ & 3/4/4 & 0.01 & 0.02 \\
Allen speed & unregularized & speed, orientation & $d_y=48$, $d_{\mathrm{res}}=96$ & 3/4/2 & 0 & 0 \\
Allen stimulus & unregularized & orientation, TF & $d_y=48$, $d_{\mathrm{res}}=96$ & 3/4/4 & 0 & 0 \\
RatInABox & pi-VAE & pos, head dir & $d_z=4$ & gen 60, flows 2 & -- & -- \\
MC Maze & pi-VAE & hand pos, hand vel, eye pos & $d_z=4$ & gen 60, flows 2 & -- & -- \\
Allen speed & pi-VAE & speed, orientation & $d_z=4$ & gen 60, flows 2 & -- & -- \\
Allen stimulus & pi-VAE & orientation, TF & $d_z=4$ & gen 60, flows 2 & -- & -- \\
MC Maze & NEDS & hand pos, hand vel, eye pos & hidden 256, inter 512 & 5 layers, seq 1 & -- & -- \\
Allen speed & NEDS & speed, orientation & hidden 256, inter 512 & 5 layers, seq 40 & -- & -- \\
Allen stimulus & NEDS & orientation, TF & hidden 256, inter 512 & 5 layers, seq 40 & -- & -- \\
\bottomrule
\end{tabular*}

\medskip
\textbf{Training settings}\par\smallskip
\begin{tabular*}{\textwidth}{@{\extracolsep{\fill}}>{\raggedright\arraybackslash}p{.105\textwidth}>{\raggedright\arraybackslash}p{.135\textwidth}>{\raggedleft\arraybackslash}p{.077\textwidth}>{\raggedright\arraybackslash}p{.135\textwidth}>{\raggedleft\arraybackslash}p{.100\textwidth}>{\raggedleft\arraybackslash}p{.060\textwidth}>{\raggedright\arraybackslash}p{.290\textwidth}@{}}
\toprule
Dataset & Model & Max. epochs & Batch & LR & WD & Schedule / seed / notes \\
\midrule
RatInABox & Neural-DISCO & 300 & 512/256/256 & $10^{-3}$ & $10^{-5}$ & $\gamma=0.7$; data seed 42; L2 warmup 10; residual off 10 \\
RatInABox & unregularized & 300 & 512/256/256 & $10^{-3}$ & $10^{-5}$ & $\gamma=0.7$; data seed 42 \\
MC Maze & Neural-DISCO & 100 & 512/256/256 & $10^{-3}$ & $10^{-5}$ & $\gamma=1$; pairing/decoder seed 0 \\
MC Maze & unregularized & 100 & 512/256/256 & $10^{-3}$ & $10^{-5}$ & $\gamma=1$; pairing/decoder seed 0 \\
Allen & Neural-DISCO & 100 & 512/256/256 & $10^{-3}$ & $10^{-5}$ & $\gamma=1$; milestones [20,50,100,300] \\
Allen & unregularized & 100 & 512/256/256 & $10^{-3}$ & $10^{-5}$ & $\gamma=1$; milestones [20,50,100,300] \\
RatInABox & pi-VAE & 600 & 256 & $5\times10^{-4}$ & -- & seed $42+\mathrm{fold}$; patience 40; 32 samples \\
MC Maze & pi-VAE & 600 & 256 & $5\times10^{-4}$ & -- & seed $42+\mathrm{fold}$; patience 40; 32 samples \\
Allen & pi-VAE & 600 & 256 & $5\times10^{-4}$ & -- & seed $42+\mathrm{fold}$; patience 40; 32 samples \\
MC Maze & NEDS & 80 & 128 & $2\times10^{-4}$ & 0.01 & seed $42+\mathrm{fold}$; patience 10; dropout 0.2/0.4 \\
Allen & NEDS & 4000 & 16 & $5\times10^{-4}$ & 0.01 & seed $42+\mathrm{fold}$; patience 50; dropout 0.2/0.4 \\
\bottomrule
\end{tabular*}
\end{table*}

Neural-DISCO models were trained with Adam and checkpoints selected by minimum validation loss. pi-VAE uses validation stopping with patience 40 and reloads the minimum-loss checkpoint. Allen NEDS retains checkpoints when validation loss improves by at least $10^{-4}$, with early-stopping patience 50, and uses training-fold input min--max normalization. The Allen stimulus-edit models use the same latent widths and optimizer settings, with factor/residual/generator depths 3/4/4 and $\lambda=.01$; their ablation sets $\lambda=\sigma=0$. The same pi-VAE and NEDS settings apply to speed and stimulus edits.

\begin{table*}[t]
\centering
\caption{Trainable parameter counts for the main editing models. Allen counts are shown for session 756029989 (672 units). The unregularized control (w/o reg.) has the same architecture as Neural-DISCO within each dataset and task and therefore has the same number of trainable parameters.}
\label{tab:param_counts}
\begin{tabular}{llr}
\toprule
Dataset & Model & Trainable params \\
\midrule
RatInABox & Neural-DISCO & 122,414 \\
RatInABox & w/o reg. & 122,414 \\
RatInABox & pi-VAE & 232,183 \\
MC Maze & Neural-DISCO & 308,772 \\
MC Maze & w/o reg. & 308,772 \\
MC Maze & pi-VAE & 64,435 \\
MC Maze & NEDS & 2,656,286 \\
Allen speed & Neural-DISCO & 621,504 \\
Allen stimulus & Neural-DISCO & 1,085,216 \\
Allen speed & w/o reg. & 621,504 \\
Allen stimulus & w/o reg. & 1,085,216 \\
Allen & pi-VAE & 1,024,968 \\
Allen & NEDS & 3,005,610 \\
\bottomrule
\end{tabular}
\end{table*}

\subsection{Baselines}
For baseline models, we used the authors' released implementations where available, adapting input and covariate dimensions and sequence length to our preprocessing. Hyperparameters and training settings for pi-VAE and NEDS are given in Table~\ref{tab:hyperparams}. Training and evaluation of the baselines across all experiments takes approximately 2 GPU hours on an H200 GPU.
\section{Proof of the residual capacity bound}
\label{app:proof}

\begin{proof}
Let $S\sim \mathrm{Unif}\{1,\ldots,N\}$, let $X_S=x^{(S)}$, and let
$U_S=\phi_{\mathrm{res}}(X_S)$. For
$\tilde U=U_S+\epsilon$ with independent
$\epsilon\sim\mathcal N(0,\sigma^2 I_m)$,
\begin{equation}
I(S;\tilde U)=h(\tilde U)-h(\tilde U\mid S)=h(\tilde U)-h(\epsilon).
\end{equation}
Let $\Sigma_U$ be the covariance of $U_S$. Then
$\mathrm{tr}(\Sigma_U)\le \mathbb E\|U_S\|_2^2=\hat P$. Since Gaussian random variables maximize differential entropy under a covariance constraint,
\begin{equation}
    h(\tilde U)\le \frac{1}{2}\log\left((2\pi e)^m\det(\Sigma_U+\sigma^2 I_m)\right).
\end{equation}
By the arithmetic--geometric mean inequality,
\begin{equation}
    \det(\Sigma_U+\sigma^2 I_m)
    \leq
    \left(
        \frac{\mathrm{tr}(\Sigma_U)+m\sigma^2}{m}
    \right)^m
    \leq
    \left(
        \sigma^2+\frac{\hat P}{m}
    \right)^m .
\end{equation}
Also $h(\epsilon)=\frac{m}{2}\log(2\pi e\sigma^2)$. Subtracting gives
\begin{equation}
    I(S;\tilde U)\le \frac{m}{2}\log\left(1+\frac{\hat P}{m\sigma^2}\right).
\end{equation}
Because $Y_{a,S}=y_a^{(S)}$ is a deterministic function of $S$, data processing gives
$I(Y_{a,S};\tilde U)\le I(S;\tilde U)$. Because
$Z_{a,S}=\phi_a(Y_{a,S})$ is a deterministic function of $Y_{a,S}$, another application of data processing gives
$I(Z_{a,S};\tilde U)\le I(Y_{a,S};\tilde U)$.
All logarithms are natural, so the mutual informations are measured in nats.
\end{proof}

\section{Supplemental Figures}
\begin{figure}[p]
  \centering
  \includegraphics[width=0.9\linewidth]{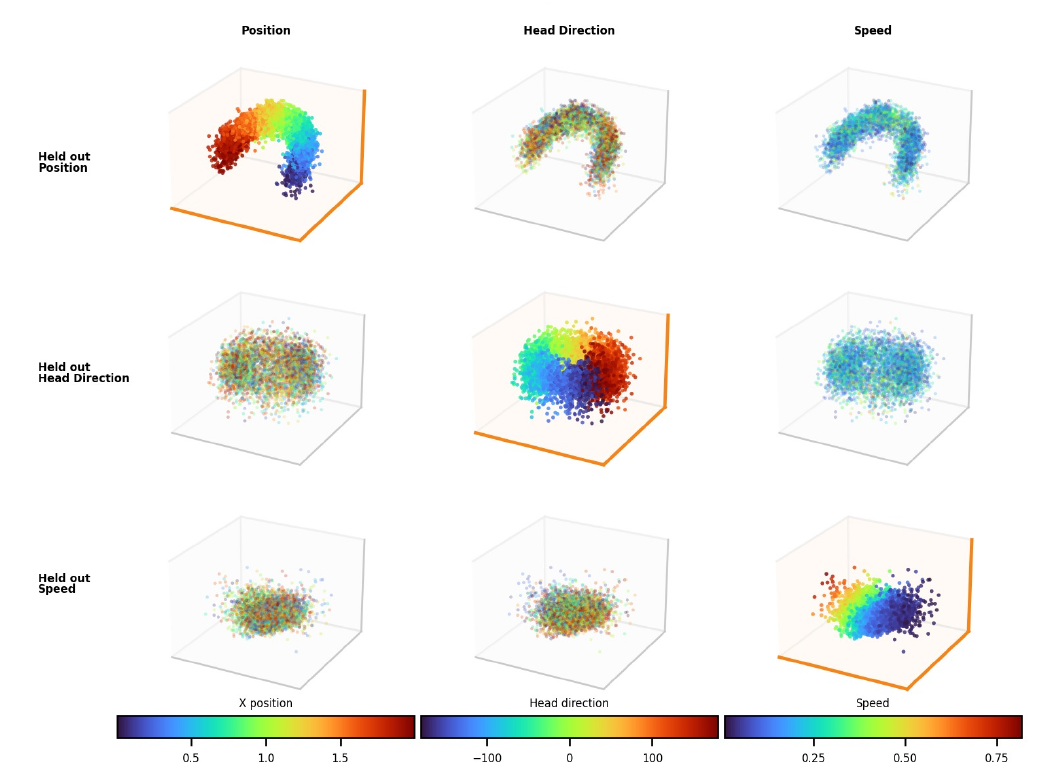}
  \caption{\textbf{Withheld-label diagnostics visualization.} Across the three held-out-label settings, the residual latent is organized by the omitted variable and not by the observed variables.}
  \label{fig:withheld_latent}
\end{figure}

\begin{figure}[t]
  \centering
  \includegraphics[width=0.6\linewidth]{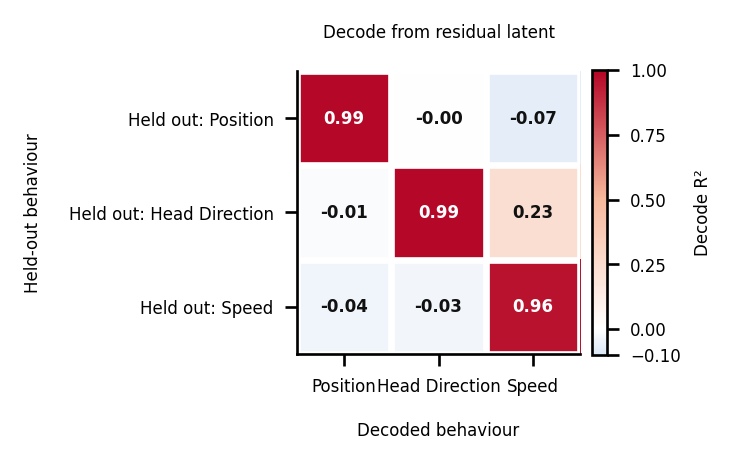}
  \caption{\textbf{Withheld-label diagnostics decoding for Rat in a Box dataset.} Residual-latent decoders recover the held-out variable much more strongly than the observed variables. }
  \label{fig:rib_withheld_decoding}
\end{figure}

\begin{figure}[t]
  \centering
  \includegraphics[width=0.6\linewidth]{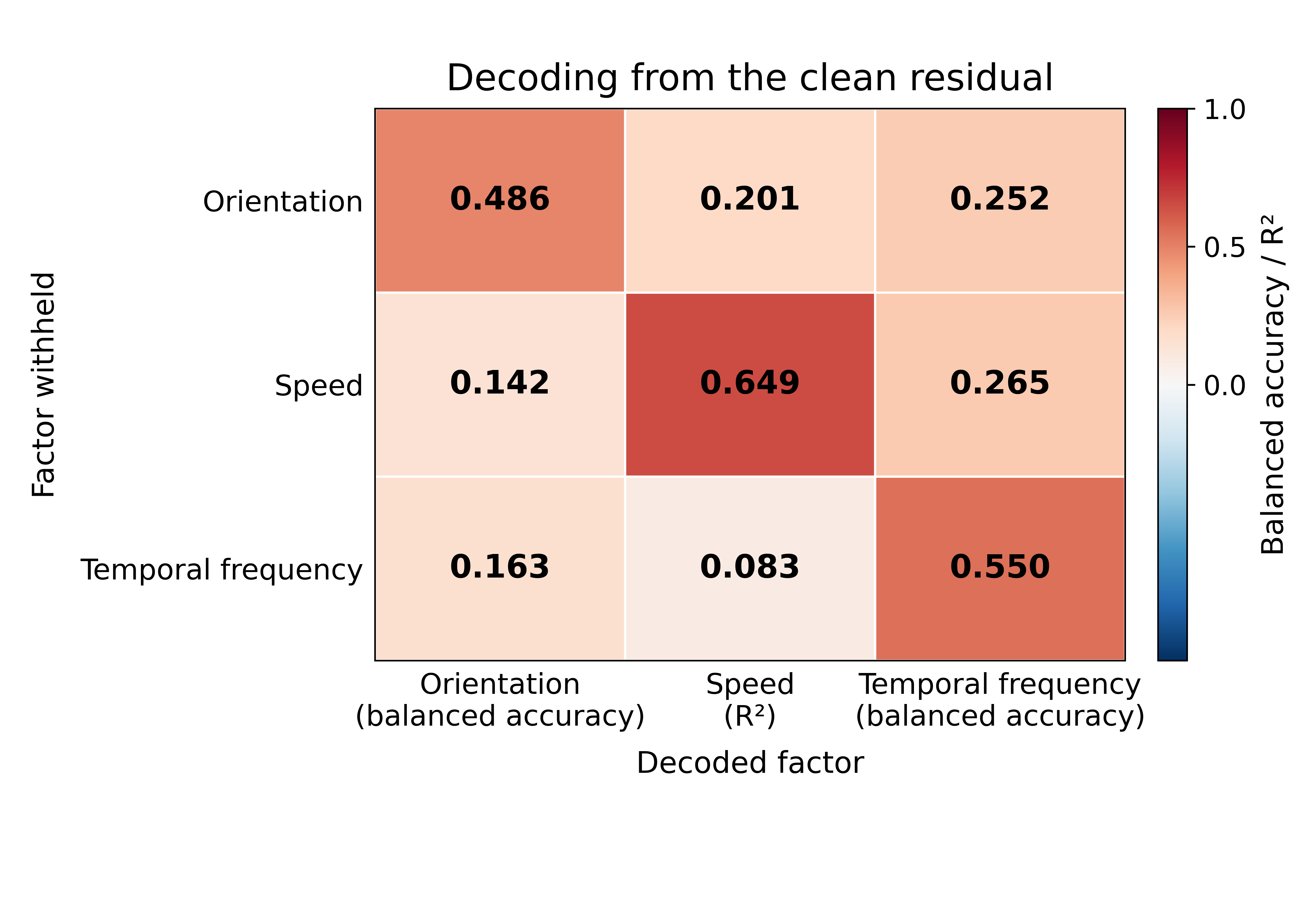}
  \caption{\textbf{Residual decoding with withheld factors on Allen.}
  Each row withholds one factor from the model's supplied variables; columns show how well each factor can be decoded from the residual latent. Chance is 0.125 for orientation, and 0.2 for temporal frequency. Each factor is more decodable when it is withheld, as reflected by the diagonal entries. The residual preferentially retains omitted-factor information, although some supplied-factor information remains (Table~\ref{tab:allen_clean_residual}).}
  \label{fig:allen_withheld_decoding}
\end{figure}

\begin{figure}[t]
  \centering
  \includegraphics[width=0.6\linewidth]{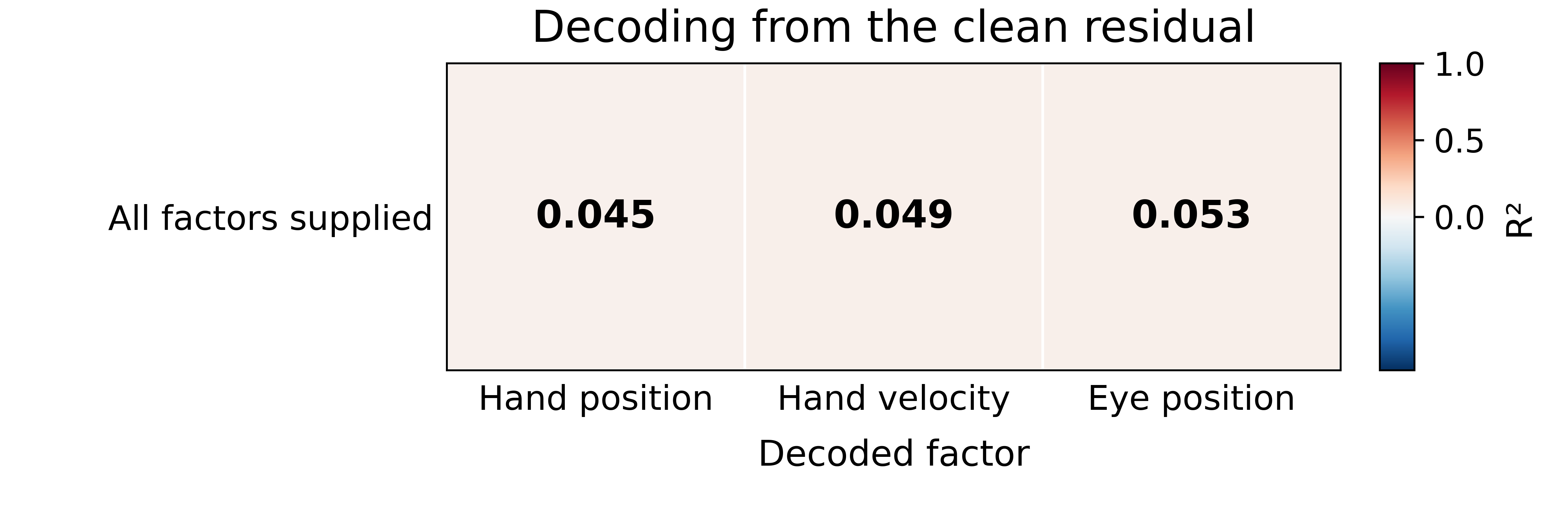}
  \caption{\textbf{Residual leakage on MC Maze.} 
  Hand position, hand velocity, and eye position are strongly correlated, as such no factor is held out; all three are supplied to the model. Leakage is therefore measured by decoding each supplied factor from the residual latent. All three factors have low held-out decoding $R^2$.}
  \label{fig:mc_maze_withheld_decoding}
\end{figure}

\begin{figure}[t]
  \centering
  \includegraphics[width=\textwidth]{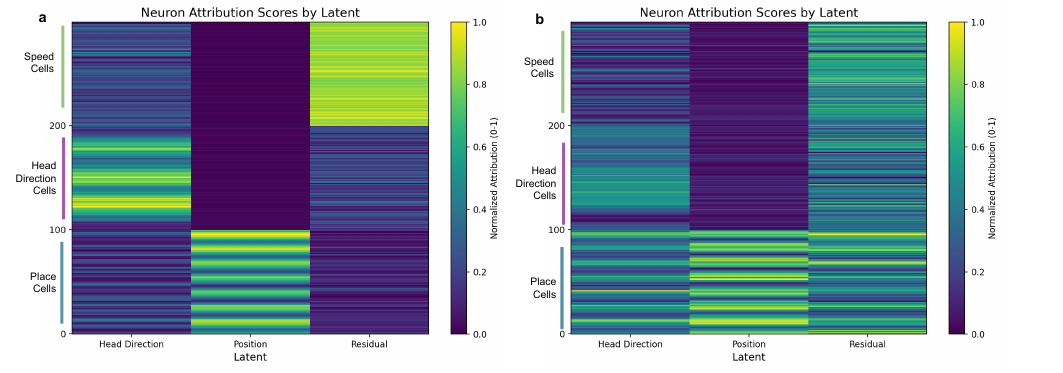}
  \caption{\textbf{RatInABox Attribution Heatmap.} Rows correspond to neurons ordered by synthetic cell type: place cells, head-direction cells, and speed cells. Columns show normalized attribution to head-direction, position, and residual blocks. (a) Attribution results of Neural-DISCO. (b) Attribution results of architecture-matched unregularized residual control.  
  }
  \label{fig:rib_heatmap}
\end{figure}

\begin{figure}[t]
  \centering
  \includegraphics[width=\textwidth]{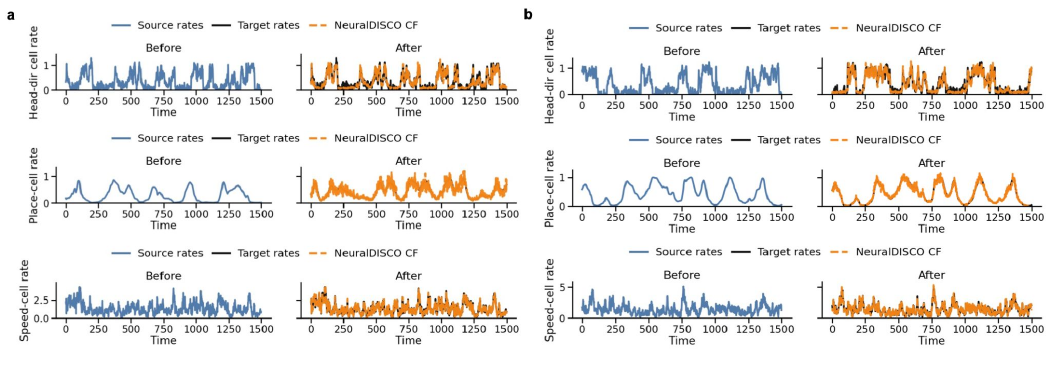}
  \caption{\textbf{RatInABox Counterfactual Rates.} Example rates for one cell of each type in (a) the position-shift edit and (b) the head-direction edit. Generated rates are compared with simulator target rates for each population.}
  \label{fig:rib_rates}
\end{figure}

\begin{figure}[t]
  \centering
  \includegraphics[width=\textwidth]{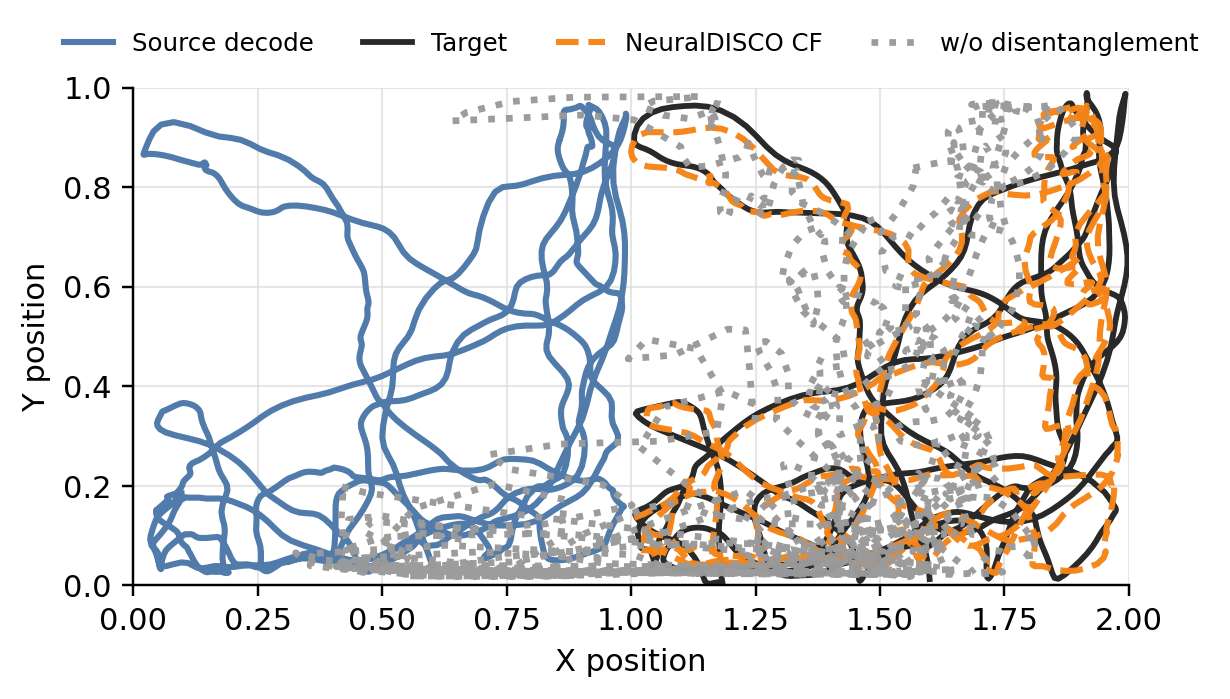}
  \caption{\textbf{RatInABox position-shift trajectory ablation.} Decoded trajectories for the position-shift counterfactual are shown for the source activity, target trajectory, the regularized Neural-DISCO counterfactual, and the architecture-matched unregularized control. The regularized model tracks the shifted target manifold substantially better, whereas the unregularized control exhibits diffuse trajectories and poorer preservation of the intended edit.}
  \label{fig:rat_position_shift_ablation}
\end{figure}

\begin{figure}[t]
  \centering
  \includegraphics[width=0.9\linewidth]{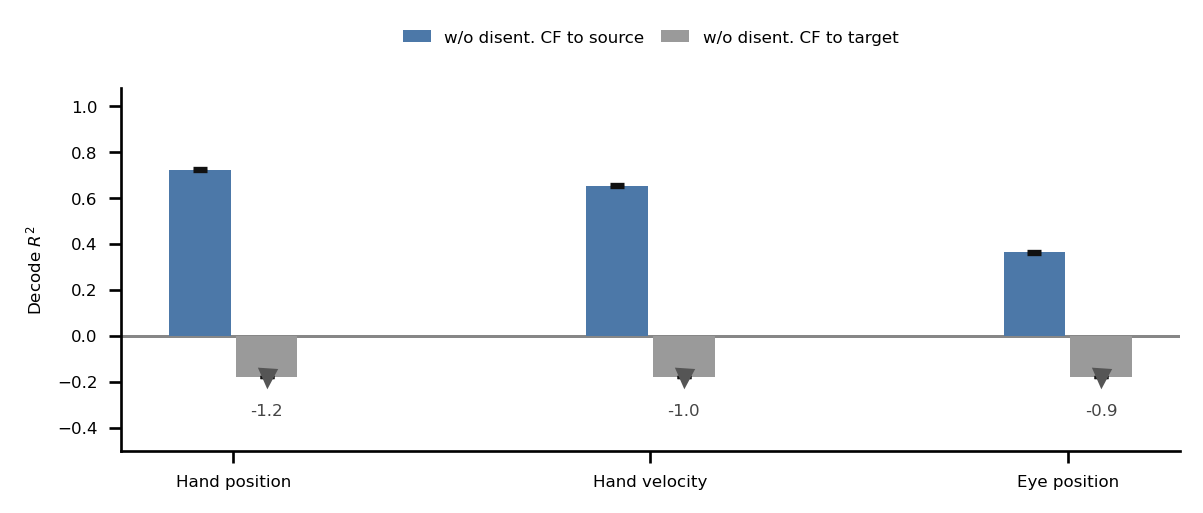}
  \caption{\textbf{MC Maze decoding for the unregularized control.} Downstream decoding $R^2$ for generated activity from the architecture-matched unregularized control, evaluated against both the source and target trajectories. For all three observed factors, the unregularized control remains closer to the source than to the target, indicating that it fails to perform the intended counterfactual edit.}
  \label{fig:mcmaze_no_reg_decode_r2}
\end{figure}

\begin{figure}[t]
  \centering
  \includegraphics[width=0.8\linewidth]{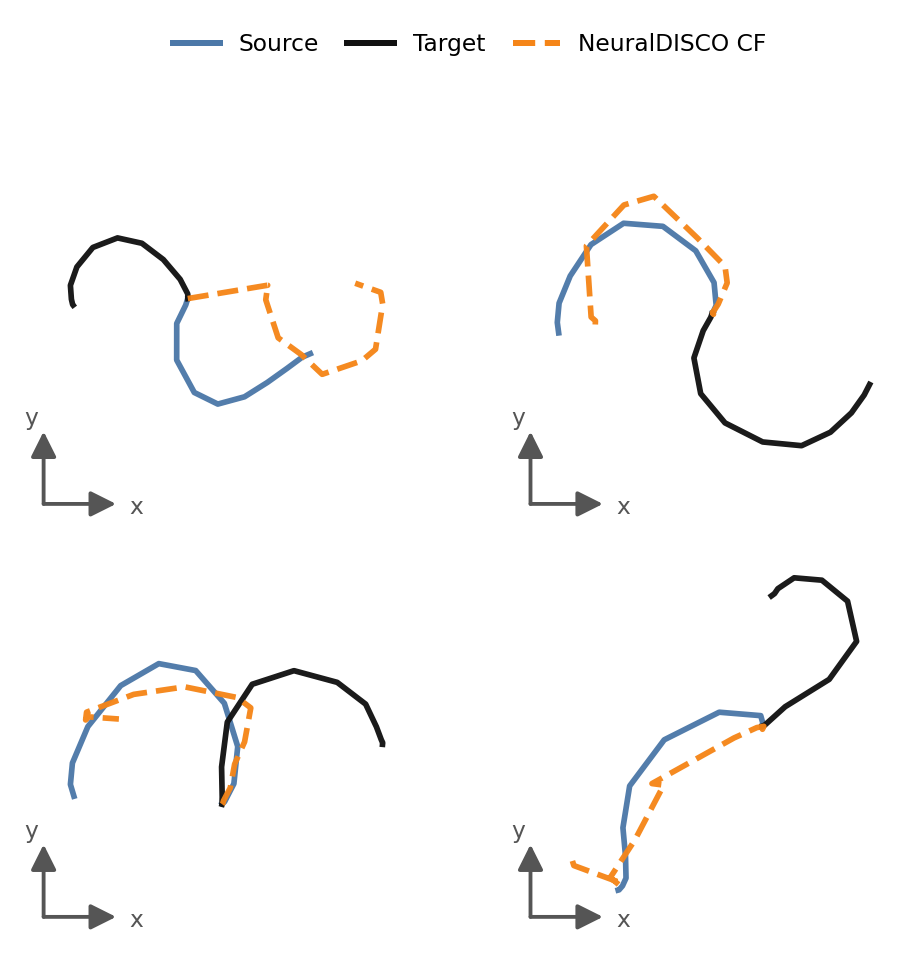}
  \caption{\textbf{MC Maze qualitative reach examples for the unregularized control.} Example hand-position trajectories for four source-target trial pairs. The generated counterfactual trajectories remain between the source and target reaches rather than matching the target movement, consistent with the poor target-decoding performance of the unregularized control.}
  \label{fig:mcmaze_no_reg_reach_trajectories}
\end{figure}

\begin{figure}[t]
  \centering
  \includegraphics[width=\textwidth]{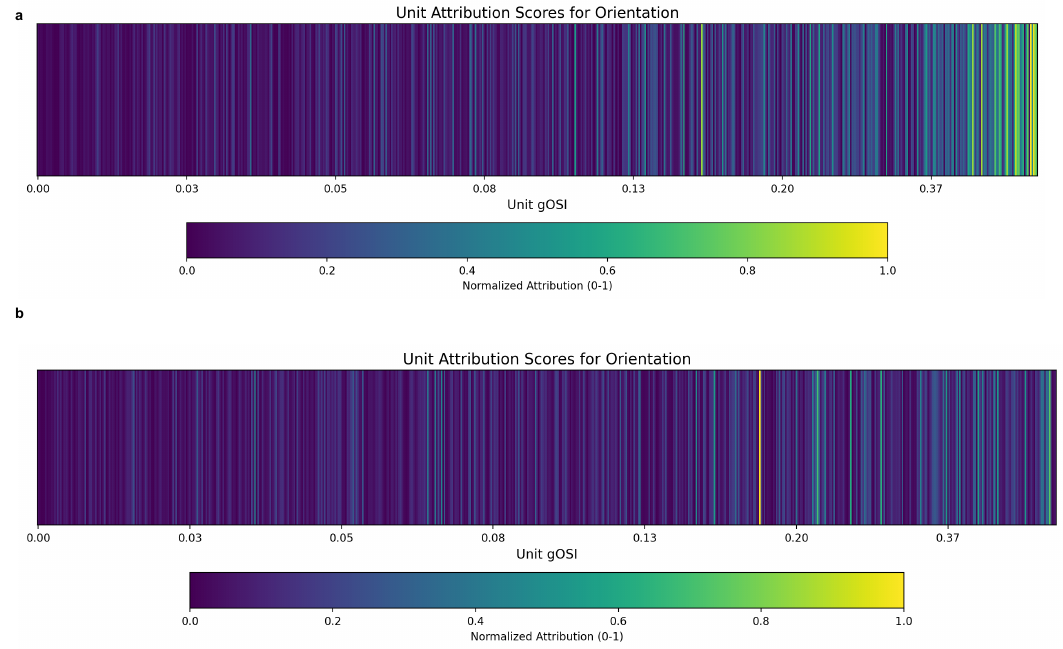}
  \caption{\textbf{Allen Attribution Heatmap sorted by Orientation Selectivity Index.} X-axis corresponds to individual units, with x-ticks indicating gOSI cutoffs, colored by orientation attribution value. Attribution scores are averaged over ten folds before min--max normalization within each method. (a) Attribution results of Neural-DISCO. (b) Attribution results of architecture-matched unregularized residual control.}
  \label{fig:allen_heatmap}
\end{figure}


\clearpage
\section{Dataset and Software Licenses}\label{appendixD}

We use the following datasets and software tools in this work. We comply with the terms of use and licensing conditions of each resource.

\textbf{RatInABox Simulator}
RatInABox \parencite{ratinabox} is an open-source Python toolkit for modelling rodent locomotion and generating synthetic neuronal data. The software is available at \url{https://github.com/RatInABox-Lab/RatInABox} and is released under the \textbf{MIT License}.
\textbf{NLB-MC\_Maze Dataset}

The MC\_Maze dataset is part of the Neural Latents Benchmark '21 \parencite{pei2022neurallatentsbenchmark21} and is publicly available via the DANDI Archive (\url{https://dandiarchive.org/dandiset/000128}). The dataset was originally collected by Kaufman, Churchland, and Shenoy at Stanford University and is distributed under the \textbf{Creative Commons Attribution 4.0 International (CC BY 4.0)} license. 
\textbf{Allen Brain Observatory -- Visual Coding Neuropixels}

The Allen Visual Coding Neuropixels dataset \parencite{AllenSDK_Neuropixels} is made available by the Allen Institute MindScope Program and is publicly accessible at \url{https://brain-map.org/our-research/circuits-behavior/visual-coding}. This dataset is distributed under the \textbf{Allen Institute Terms of Use} (\url{https://www.alleninstitute.org/legal/terms-of-use}), which permits free use for non-commercial research and educational purposes with attribution.

\textbf{pi-VAE}. We used the official implementation of pi-VAE \parencite{pivae}, available at \url{https://github.com/zhd96/pi-vae}.

\textbf{NEDS}. We used the official implementation of NEDS \parencite{neds}, available at
\url{https://github.com/yzhang511/NEDS}, which is released under the MIT License.

\textbf{Captum Integrated Gradients}. We use Captum's implementation of Integrated Gradients, available at \href{https://captum.ai/docs/extension/integrated_gradients}{https://captum.ai/docs/extension/integrated\_gradients}


\clearpage
\section*{NeurIPS Paper Checklist}

\begin{enumerate}

\item \textbf{Claims}
\item[] Question: Do the main claims made in the abstract and introduction accurately reflect the paper's contributions and scope?
\item[] Answer: \answerYes{}.
\item[] Justification: The abstract and introduction frame the paper around source-conditioned counterfactual editing and explicitly state that the counterfactuals are model-based predictions rather than identified causal effects.

\item \textbf{Limitations}
\item[] Question: Does the paper discuss the limitations of the work performed by the authors?
\item[] Answer: \answerYes{}.
\item[] Justification: Section~\ref{sec:discussion} discusses causal interpretation, residual-bottleneck sensitivity, and indirect validation in real neural data.

\item \textbf{Theory assumptions and proofs}
\item[] Question: For each theoretical result, does the paper provide the full set of assumptions and a complete and correct proof?
\item[] Answer: \answerYes{}.
\item[] Justification: Proposition~1 states the finite-sample residual capacity bound and Appendix~\ref{app:proof} gives the proof and assumptions.

\item \textbf{Experimental result reproducibility}
\item[] Question: Does the paper fully disclose all the information needed to reproduce the main experimental results of the paper to the extent that it affects the main claims and/or conclusions of the paper, regardless of whether the code and data are provided or not?
\item[] Answer: \answerYes{}.
\item[] Justification: A complete description of model structure and hyperparameters is provided in Appendix~\ref{app:details}, along with preprocessing instructions, random seeds, fold definitions, and session/unit filtering details needed for full reproduction.

\item \textbf{Open access to data and code}
\item[] Question: Does the paper provide open access to the data and code, with sufficient instructions to faithfully reproduce the main experimental results, as described in supplemental material?
\item[] Answer: \answerYes{}.
\item[] Justification: Our experiments use public datasets and simulators that are publicly available and we provide exact instructions for acquiring/setting them up. As well, the manuscript package includes a \href{https://github.com/TangLab-UBC/Neural-DISCO-neurips}{code release} with end-to-end reproduction commands/scripts for exact reproducibility.

\item \textbf{Experimental setting/details}
\item[] Question: Does the paper specify all the training and test details necessary to understand the results?
\item[] Answer: \answerYes{}.
\item[] Justification: The paper includes core training settings and hyperparameters in Appendix~\ref{app:details}, as well as all configuration parameters and information, including optimizer settings, preprocessing commands, split definitions, session identifiers, and filtering criteria.

\item \textbf{Experiment statistical significance}
\item[] Question: Does the paper report error bars suitably and correctly defined or other appropriate information about the statistical significance of the experiments?
\item[] Answer: \answerYes{}.
\item[] Justification: Section~\ref{sec:experiments} states that reported $\pm$ values and error bars denote standard error of the mean across CV models.

\item \textbf{Experiments compute resources}
\item[] Question: For each experiment, does the paper provide sufficient information on the computer resources needed to reproduce the experiments?
\item[] Answer: \answerYes{}.
\item[] Justification: Appendix~\ref{app:details} reports the GPU type, memory, approximate wall-clock GPU hours, and per-experiment compute breakdown.

\item \textbf{Code of ethics}
\item[] Question: Does the research conducted in the paper conform, in every respect, with the NeurIPS Code of Ethics?
\item[] Answer: \answerYes{}.
\item[] Justification: The work uses public animal neural datasets and a simulator, and the manuscript does not include human-subject data collection or high-risk deployment.

\item \textbf{Broader impacts}
\item[] Question: Does the paper discuss both potential positive societal impacts and negative societal impacts of the work performed?
\item[] Answer: \answerYes{}.
\item[] Justification: Section~\ref{sec:discussion} discusses the expected broader impact of the work as a research tool for \textit{in-silico} hypothesis generation in computational neuroscience. The work uses public animal neural recordings and simulations, introduces no new human-subject dataset, and is not intended for clinical, diagnostic, or closed-loop interventional deployment. Its broader societal risk is therefore limited, while its primary impact is to help prioritize neural--behavioral hypotheses and guide follow-up analyses.

\item \textbf{Safeguards}
\item[] Question: Does the paper describe safeguards that have been put in place for responsible release of data or models that have a high risk for misuse?
\item[] Answer: \answerNA{}.
\item[] Justification: The work does not release a high-risk generative model such as a language, image, or scraped-data model; it targets animal neural activity modeling.

\item \textbf{Licenses for existing assets}
\item[] Question: Are the creators or original owners of assets used in the paper properly credited and are the license and terms of use explicitly mentioned and properly respected?
\item[] Answer: \answerYes{}.
\item[] Justification: The paper cites the public datasets and baseline methods. Appendix~\ref{appendixD} lists asset sources and licenses; the code package provides pinned software dependencies and environment installation instructions.

\item \textbf{New assets}
\item[] Question: Are new assets introduced in the paper well documented and is the documentation provided alongside the assets?
\item[] Answer: \answerYes{}.
\item[] Justification: The paper does not introduce a new dataset. The code package includes reproduction instructions, experiment configurations, and pinned dependency locks.

\item \textbf{Crowdsourcing and research with human subjects}
\item[] Question: For crowdsourcing experiments and research with human subjects, does the paper include the full text of instructions given to participants and screenshots, if applicable, as well as details about compensation?
\item[] Answer: \answerNA{}.
\item[] Justification: The reported experiments do not involve crowdsourcing or new human-subject data collection.

\item \textbf{Institutional review board approvals or equivalent for research with human subjects}
\item[] Question: Does the paper describe potential risks incurred by study participants, whether such risks were disclosed to the subjects, and whether Institutional Review Board approvals or an equivalent review were obtained?
\item[] Answer: \answerNA{}.
\item[] Justification: The reported experiments do not involve new human-subject data collection.

\item \textbf{Declaration of LLM usage}
\item[] Question: Does the paper describe the usage of LLMs if it is an important, original, or non-standard component of the core methods in this research?
\item[] Answer: \answerNA{}.
\item[] Justification: The core methodology does not use LLMs as a model component. Use of LLMs only for writing or formatting does not need to be declared under this question.

\end{enumerate}

\end{document}